\documentclass[a4paper,12pt]{article}

\usepackage{lmodern}
    \PassOptionsToPackage{numbers, compress}{natbib}

\usepackage[top=3.25cm, bottom=3.25cm, left=2.5cm, right=2.5cm]{geometry}

\usepackage[utf8]{inputenc} 
\usepackage[T1]{fontenc}    
\RequirePackage[colorlinks,citecolor=blue,urlcolor=blue]{hyperref}

\usepackage{url}            
\usepackage{booktabs}       
\usepackage{amsfonts}       
\usepackage{nicefrac}       
\usepackage{microtype}      
\usepackage{times}
\usepackage[english]{babel}
\usepackage{amsmath,amssymb,amsthm} 
\usepackage{algorithm}
\usepackage{algpseudocode}
\usepackage{appendix}
\usepackage{thmtools,thm-restate}
\usepackage[nameinlink,capitalize]{cleveref}
\usepackage{etoolbox}
\usepackage{xcolor}
\usepackage{todonotes}
\usepackage{multicol}

\usepackage[sort&compress]{natbib}
\usepackage{accents}
\newcommand\thickbar[1]{\accentset{\rule{.365em}{.5pt}}{#1}}

\newcommand{\X}{\mathcal{X}}
\newcommand{\Y}{\mathcal{Y}}
\newcommand{\Z}{\mathcal{Z}}

\newcommand{\A}{\mathcal{A}}

\renewcommand{\H}{\mathcal{H}}

\newcommand{\trans}{^{\scriptscriptstyle \top}}
\newcommand{\R}{\mathbb{R}}
\newcommand{\N}{\mathbb{N}}

\newcommand{\la}{\lambda}

\newcommand{\tr}{\ensuremath{\text{\rm Tr}}}

\newcommand{\argmin}{\operatornamewithlimits{argmin}}

\newcommand{\ee}{{\mathcal{E}}}

\newcommand{\rx}{{R}}

\newcommand{\Zn}{Z}

\newcommand{\LL}{\mathcal{L}}

\newcommand{\cR}{\mathcal{R}}
\newcommand{\T}{\mathcal{T}}
\newcommand{\Real}{\mathbb{R}}
\newcommand{\Exp}{\mathbb{E}}

\newcommand{\vertiii}[1]{{\left\vert\kern-0.25ex\left\vert\kern-0.25ex\left\vert #1 
    \right\vert\kern-0.25ex\right\vert\kern-0.25ex\right\vert}}

\declaretheorem[name=Theorem,refname=Thm.]{theorem}

\declaretheorem[name=Lemma,sibling=theorem]{lemma}
\declaretheorem[name=Proposition,refname=Prop.,sibling=theorem]{proposition}
\declaretheorem[name=Remark]{remark}

\declaretheorem[name=Assumption,refname=Asm.]{assumption}
\declaretheorem[name=Example]{example}

 \makeatletter
 \renewenvironment{proof}[1][\proofname]{\par
   \pushQED{\qed}
   \normalfont \topsep6\p@\@plus6\p@\relax
   \trivlist
   \item[\hskip\labelsep
         \bfseries
     #1\@addpunct{.}]\ignorespaces
 }{
   \popQED\endtrivlist\@endpefalse
 }
 \makeatother
\def\eop{$\rule{1.3ex}{1.3ex}$}
\renewcommand\qedsymbol\eop

\AfterEndEnvironment{restatable}{\noindent\ignorespacesafterend}
\AfterEndEnvironment{theorem}{\noindent\ignorespacesafterend}
\AfterEndEnvironment{remark}{\noindent\ignorespacesafterend}
\AfterEndEnvironment{example}{\noindent\ignorespacesafterend}
\AfterEndEnvironment{assumption}{\noindent\ignorespacesafterend}
\AfterEndEnvironment{lemma}{\noindent\ignorespacesafterend}
\AfterEndEnvironment{proof}{\noindent\ignorespacesafterend}
\AfterEndEnvironment{corollary}{\noindent\ignorespacesafterend}
\AfterEndEnvironment{proposition}{\noindent\ignorespacesafterend}
\AfterEndEnvironment{definition}{\noindent\ignorespacesafterend}

\newcommand{\task}{\mu}  
\newcommand{\Ss}{\mathcal{S}}  
\newcommand{\Tt}{\mathcal{M}}   
\newcommand{\D}{\mathcal{D}}  
\newcommand{\env}{\rho}  
\newcommand{\ms}{{\env_\Ss}}  
\newcommand{\mt}{{\env_\Tt}}  

\crefname{assumption}{Asm.}{Asm.}
\crefname{equation}{}{}
\crefname{figure}{Fig.}{Fig.}
\crefname{table}{Tab.}{Tab.}
\crefname{section}{Sec.}{Sec.}
\crefname{theorem}{Thm.}{Thm.}
\crefname{proposition}{Prop.}{Prop.}
\crefname{fact}{Fact}{Facts}
\crefname{lemma}{Lemma}{Lemmas}
\crefname{corollary}{Cor.}{Cor.}
\crefname{example}{Ex.}{Ex.}
\crefname{remark}{Rem.}{Rem.}
\crefname{algorithm}{Alg.}{Algorithms}
\crefname{appendix}{App.}{App.}
\crefname{algorithm}{Alg.}{Alg.}

\providecommand{\scal}[2]{\left\langle{#1},{#2}\right\rangle}

\providecommand{\nor}[1]{\left\|{#1}\right\|}
\renewcommand{\paragraph}[1]{{\bfseries #1.}}
\renewcommand{\P}{{\mathcal{P}}}

\newcommand{\hh}{{\mathcal{H}}}
\providecommand{\scal}[2]{\left\langle{#1},{#2}\right\rangle}

\providecommand{\nor}[1]{\left\|{#1}\right\|}

\title{\bf The Advantage of Conditional Meta-Learning for 
Biased Regularization and Fine-Tuning
\vspace{1.0truecm}
}

\author{Giulia Denevi$^{1}$ ~~~~~~~~~~~~~ 
Massimiliano Pontil$^{1,2}$ ~~~~~~~~~~~~~ Carlo Ciliberto$^{3}$ \\ 
{\small \hspace*{-2.0em} ~~~~~~~~~~~~ giulia.denevi@iit.it ~~~~~~~~~~~~~~~~~ massimiliano.pontil@iit.it ~~~~~~~~~~~~~~~~~ c.ciliberto@imperial.ac.uk} \vspace{.1cm} \\
\small{$^1$ Computational Statistics and Machine Learning, Istituto Italiano di Tecnologia, Genova, Italy} \\
\small{$^2$ Computer Science Dept., University College of London, London, United Kingdom} \\
\small{$^3$ Electrical and Electronic Engineering Dept., Imperial College of London, London, United Kingdom}}

\begin{document}

\maketitle

\begin{abstract}
\noindent Biased regularization and fine-tuning are two recent meta-learning approaches. They have been shown to be effective to tackle distributions of tasks, in which the tasks' target vectors are all close to a common meta-parameter vector.
However, these methods may perform poorly on heterogeneous environments of tasks,  where the complexity of the tasks' distribution cannot be captured by a single meta-parameter vector. We address this limitation by conditional meta-learning, inferring a  conditioning function mapping task's side information into a meta-parameter vector that is appropriate for that task at hand. We characterize properties of the environment under which the conditional approach brings a substantial advantage over standard meta-learning and we highlight examples of environments, such as those with multiple clusters, satisfying these properties. We then propose a convex meta-algorithm providing a comparable advantage also in practice. Numerical experiments confirm our theoretical findings.
\end{abstract}


\section{Introduction}
\label{Introduction}

Biased regularization and fine-tuning 
\cite{pmlr-v70-finn17a,finn2019online,nichol2018first, maurer2009transfer,pentina2014pac,denevi2019learning,denevi2019online,balcan2019provable,khodak2019adaptive,ji2020multi,fallah2019convergence} 
are two recent meta-learning techniques that transfer knowledge across an environment of tasks by leveraging a common meta-parameter vector. Their origin and inspiration go back to multi-task and transfer learning methods \cite{evgeniou2004regularized,maurer2006,cavallanti2010linear}, designed to address a prescribed set of tasks with low variance. These techniques can be described as a nested optimization scheme: while at the within-task level, an inner algorithm performs tasks' specific optimization with the current meta-parameter vector, at the meta-level a meta-algorithm updates the aforementioned meta-parameter by leveraging the experience accumulated from the tasks observed so far. In biased regularization the inner algorithm is given by the within-task regularized empirical risk minimizer and the meta-parameter vector plays the role of a bias in the regularizer, while fine-tuning employs online gradient descent as the within-task algorithm and the meta-parameter vector is the associated starting point.

Despite their success, the above methods may fail to adapt to heterogenous 
environments of tasks, in which the complexity of the tasks' distribution cannot 
be captured by a single meta-parameter vector. 
In literature, a variety of methods have tried to address this limitation by clustering the tasks and, then, leveraging tasks' similarities within each cluster
\cite{argyriou2008algorithm,maurer2012transfer,argyriou2013learning,jacob2009clustered,Andrew}.
However, such methods usually lead to non-convex formulations 
\cite{argyriou2008algorithm,argyriou2013learning}  
or provide only partial guarantees on surrogate convex problems \cite{jacob2009clustered,Andrew}. As alternative, recent approaches 
in meta-learning literature advocated learning a conditioning function that maps a task's dataset into a meta-parameter vector that is appropriate for the task at hand~\cite{wang2020structured,vuorio2019multimodal,rusu2018meta,jerfel2019reconciling,cai2020weighted,yao2019hierarchically}. 
This perspective has been shown to be promising in applications, however theoretical investigations are still lacking. 
In this work, we address the limitation above for biased regularization and 
fine-tuning by developing a new conditional meta-learning framework.
Specifically, we consider an environment of tasks provided 
with additional side information and we learn a conditioning function mapping 
task's side information into a task's specific meta-parameter vector. We then
provide a statistical analysis demonstrating the potential advantage 
of our method over standard meta-learning. \\

\paragraph{Contributions and organization} Our work offers four contributions. First, in \cref{unconditional_for_bias_introduction}, we introduce a new conditional 
meta-learning framework with side information for biased regularization and 
fine-tuning. Second, in \cref{unconditional_for_bias_advantage}, we formally 
show that, under certain assumptions, this conditional meta-learning approach results 
to be significantly advantageous w.r.t. the standard unconditional counterpart. 
We then describe two common settings in which such conditions are 
satisfied, supporting the potential importance of our study for real-world scenarios.
Third, in \cref{proposed_method}, we propose a convex meta-algorithm 
providing a comparable advantage also in practice, as the number of observed 
tasks increases. Fourth, in \cref{experiments}, we present numerical experiments in 
which we test our theory and the performance of our method. Our conclusions
are drawn in \cref{conclusion} and technical proofs are postponed to the appendix.


\section{Conditional meta-learning}
\label{unconditional_for_bias_introduction}

In this section we describe and contrast the conditional meta-learning setting 
with side information to standard meta-learning. We first introduce the class 
of inner algorithms we consider in this work.\\

\paragraph{Inner algorithms (linear supervised learning)}
Let $\Z = \X\times\Y$ with $\X\subseteq\R^d$ and $\Y\subseteq\R$ 
input and output spaces, respectively. Let $\P(\Z)$ be the set of probability 
distributions (tasks) over $\Z$. Given $\mu \in {\cal P}(\Z)$ and a loss 
function $\ell:\R\times\R\to\R$, our goal is to find a weight vector 
$w_\task\in\R^d$ minimizing the {\itshape expected risk}
\begin{equation} \label{single_task_risk}
\min_{w\in\R^d}~\cR_\task(w) 
\quad \quad \quad
\cR_\task(w) = \Exp_{(x,y)\sim\task}~\ell \bigl(\scal{ x}{w}, y \bigr),
\end{equation}
where, $\scal{x}{w}$ denotes the standard inner product between $x$ and $w\in\R^d$. 
In practice, $\task$ is unknown and only accessible trough a training dataset 
$\Zn = (x_i, y_i)_{i = 1}^n \sim \task^n$ of i.i.d. (identically independently 
distributed) points sampled from $\Zn$. The goal of a learning algorithm is to find a candidate weight vector incurring a small expected risk converging to the ideal $\cR_\mu(w_\mu)$ as $n$ grows.

In this work we will focus on the family of learning algorithms 
performing biased regularized empirical risk minimization. 
Formally, given $\D = \bigcup_{n\in\N}\Z^n$ the space of all datasets (of any 
finite cardinality $n$) on $\Z$ and a bias vector $\theta\in\Theta=\R^d$, we will consider learning algorithms $A(\theta, \cdot):\D\to\R^d$ such that,
\begin{equation} \label{RERM_bias}
A(\theta, \Zn) = \argmin_{w \in \Real^d} ~ \cR_\Zn^\la(w)
\quad \quad \quad 
\cR_\Zn^\la(w) = \frac{1}{n} \sum_{i = 1}^n \ell(\langle x_i, w \rangle, y_i) 
+ \frac{\la}{2} \| w - \theta \|^2,
\end{equation}
for any $Z = (x_i,y_i)_{i=1}^n$. Here $
\nor{\cdot}$ denotes the Euclidean norm on $\R^d$ and $\la>0$ is a regularization parameter encouraging the algorithm $A(\theta,\cdot)$ to predict weight vectors that are close to $\theta$. 
We denote by $\cR_Z(\cdot) = \cR_Z^0(\cdot)$ the empirical risk associated to $\Zn$. 
\begin{remark}[Fine-tuning] \label{online_algorithm_remark}
In this work we primarily focus on the family \cref{RERM_bias} of batch inner algorithms. However, following  \cite{denevi2019learning, denevi2019online}, it is 
possible to extend our analysis to fine-tuning algorithms performing online gradient descent on $\cR_Z^\la$, with starting point $w_1 = \theta\in\R^d$, namely
\begin{equation} \label{online_inner_algorithm}
A(\theta, \Zn) = \frac{1}{n} \sum_{i = 1}^n w_i,
\qquad 
w_{i+1} = w_i - \frac{s_i x_i - \la (w_i - \theta) }{\la i},
\qquad s_i \in \partial \ell(\cdot, y_i)(\langle x_i, w_i \rangle).
\end{equation}
\end{remark}

\paragraph{(Unconditional) meta-learning} 
Given a meta-distribution $\rho  \in \P(\Tt)$ (or {\itshape environment} \cite{baxter2000model}) 
over a family $\Tt\subseteq\P(\Z)$ of distributions (tasks) $\task$, 
meta-learning aims to learn an inner algorithm in the family that is {well suited to 
tasks $\task$ sampled from $\env$}. This goal can be reformulated 
as finding a meta-parameter $\theta_\env \in \Theta$ whose associated 
algorithm $A(\theta_\env,\cdot)$ minimizes the \emph{transfer risk}
\begin{equation} \label{meta_learning_problem_1}
\min_{\theta \in \Theta} ~ \ee_\env(\theta)
\quad \quad \quad 
\ee_\env(\theta) = \Exp_{\task \sim \env} ~ \Exp_{\Zn \sim \task^n}
~ \cR_\task \bigl( A(\theta, \Zn) \bigr).
\end{equation}
Standard meta-learning methods \cite{pmlr-v70-finn17a,finn2019online,denevi2019learning,denevi2019online,balcan2019provable,khodak2019adaptive}  
usually address this problem via stochastic methods. They
iteratively sample a task $\task\sim\env$ and a dataset $\Zn \sim \task^n$, 
and, then, they perform a step of stochastic 
gradient descent on a surrogate problem of \cref{meta_learning_problem_1} 
computed by using $\Zn$.

Although remarkably effective in many applications 
\cite{finn2019online,balcan2019provable,khodak2019adaptive,denevi2019online,
pmlr-v70-finn17a,denevi2019learning}, the framework above implicitly assumes 
that a single bias vector is sufficient for the entire family of tasks sampled 
from $\env$. Since this assumption may not hold for more complex meta-distributions 
(e.g. multi-clusters), recent works have advocated a conditional perspective to tackle this 
problem \cite{wang2020structured,vuorio2019multimodal,rusu2018meta,jerfel2019reconciling,cai2020weighted,yao2019hierarchically}.
\\

\paragraph{Conditional meta-learning}
Assume now that when sampling a task $\task$, we are also given additional side information $s \in \Ss$ to help solving the task. Within this setting the environment corresponds to a distribution $\rho \in \P(\Tt,\Ss)$ over the set $\Tt$ of tasks and the set $\Ss$ of possible side information.
The notion of side information is general, and recovers settings where $s$ contains descriptive features associated to a task (e.g. attributes in collaborative filtering \cite{abernethy2009new}) or $s$ is an additional dataset sampled from $\task$ (see \cite{wang2020structured} or \cref{double_sample_size} below).
Intuitively, meta-learning might solve a new task better if it was able to leverage 
this additional side information. We formalize this concept by adapting (or conditioning) 
the meta-parameters $\theta \in \Theta$ on the side information $s \in \Ss$, by learning 
a meta-parameter-valued function $\tau$ minimizing
\begin{equation} \label{conditional_meta_learning_problem_1}
\min_{\tau \in \T} ~ \ee_\env(\tau)
\quad \quad \quad
\ee_\env(\tau) = \Exp_{(\task, s) \sim \env} ~ \Exp_{\Zn \sim \task^n} 
~ \cR_\task \bigl( A(\tau(s), \Zn) \bigr),
\end{equation}
over the space $\T$ of measurable functions $\tau: \Ss \to \Theta$. Note that the unconditional meta-learning problem in \cref{meta_learning_problem_1} is retrieved
by restricting \cref{conditional_meta_learning_problem_1} to $\T^{\rm const} = \{\tau ~|~ \tau(\cdot) \equiv \theta,~ \theta\in\Theta\}$, the set of constant functions associating any side information to a fixed bias vector. We assume $\env$ to decompose in $\env(\cdot|s)\env_\Ss(\cdot)$ and $\env(\cdot|\mu)\env_\Tt(\cdot)$ the conditional and marginal distributions 
w.r.t. (with respect to) $\Ss$ and $\Tt$. 
In the following, we will quantify the benefits of adopting the conditional perspective 
above and, then, we propose an efficient algorithm to address \cref{conditional_meta_learning_problem_1}. We conclude this section by drawing a 
connection between our formulation and previous work on the topic.

\begin{remark}[Datasets as side information] \label{double_sample_size}
A relevant setting is the case where the side information $s$ corresponds to
an additional ({\itshape conditional}) dataset $\Zn^{cond}$ sampled from $\mu$, 
as proposed in \cite{wang2020structured}. We note however that our sampling 
scheme in \cref{conditional_meta_learning_problem_1} implies that side 
information $s$ and training set $\Zn$ are independent conditioned on $\task$. 
Hence, our framework does not allow having $s = \Zn^{cond} = \Zn$, namely, 
to use the same dataset for both conditioning and training the inner algorithm $A(\tau(\Zn),\Zn)$, as done in \cite{wang2020structured}.
This is a minor issue since one can always split $\Zn$ in two parts and use one 
part for training and the other one for conditioning.
\end{remark}


\section{The advantage of conditional meta-learning}
\label{unconditional_for_bias_advantage}

In this section we study the generalization properties of a given
conditional function $\tau$. This will allow us to characterize the behavior 
of the ideal solution of  \cref{conditional_meta_learning_problem_1} and to 
illustrate the potential advantage of conditional meta-learning. Specifically, 
we wish to estimate the error $\ee_\env(\tau)$ w.r.t. the ideal risk
\begin{equation} \label{oracle}
\ee_\env^* = \Exp_{\task \sim \rho} ~ \cR_\task (w_\task)
\quad \quad \quad 
w_\task = \argmin_{w \in \Real^d} ~ \cR_\task (w).
\end{equation}
For any $\tau \in \T$ the following quantity will play a central role 
in our analysis:
\begin{equation} \label{conditional_var}
{\rm Var}_\env(\tau)^2 = \Exp_{(\task, s) \sim \env} ~ 
\big \| w_\task - \tau(s) \big \|^2.
\end{equation}
With some abuse of terminology, we refer to ${\rm Var}_\env(\tau)$ as the {\itshape variance} of $w_\task$ w.r.t. $\tau$ (it corresponds to the actual variance of $w_\task$ when $\tau$ is the minimizer, see \cref{oracle_function} below).
Under the following assumption, we can control the excess risk of $\tau$ in terms of 
${\rm Var}_\env(\tau)$.

\begin{assumption}\label{ass_1}
Let $\ell$ be a convex and $L$-Lipschitz loss function in the first argument. Additionally, there exist $\rx>0$ such that $\| x \| \le \rx$ for any $x\in\X$. 
\end{assumption}

\begin{theorem}[Excess risk with generic conditioning function $\tau$] \label{bound_fixed_bias}
Let \cref{ass_1} hold. Given $\tau \in \T$, let $A(\theta,\cdot)$ be the generic inner algorithm in \cref{RERM_bias} with regularization parameter $\la = 2 L \rx {\rm Var}_\env(\tau)^{-1} 
n^{- 1/2}$. Then,
\begin{equation}
\ee_\env(\tau) - \ee_\env^* ~ \le ~ \frac{2 \rx L ~ {\rm Var}_\env(\tau)}{n^{1/2}}.
\end{equation}
\end{theorem}
\begin{proof}
We consider the decomposition
$\ee_\env(\tau) - \ee_\env^* = \Exp_{(\task, s) \sim \env} 
\big[ \text{B}_{\task,s} + \text{C}_{\task,s} \big]$, with
\begin{equation}
\text{B}_{\task,s} = 
\Exp_{\Zn \sim \task^n} ~ 
\Big [ \cR_\task(A(\tau(s), \Zn))  - \cR_\Zn(A(\tau(s), \Zn)) \Big ]
\end{equation}
\begin{equation} \label{approximation_error}
\text{C}_{\task,s} = 
\Exp_{\Zn \sim \task^n} ~ \Big [ \cR_\Zn(A(\tau(s), \Zn)) - \cR_\task (w_\task) \Big] \le 
\Exp_{\Zn \sim \task^n} ~ \Big [ \min_{w \in \Real^d} 
~ \cR_\Zn^\la(w) - \cR_\task (w_\task) \Big].
\end{equation}
$\text{B}_{\task,s}$ is the generalization error of the inner algorithm 
$A(\tau(s),\cdot)$ on the task $\mu$. 
Hence, applying \cref{ass_1} and the stability arguments in 
\cref{generalization_error_RERM} in \cref{generalization_error_RERM_sec}, 
we can write 
$\text{B}_{\task,s} \le 2 \rx^2 L^2 (\la n)^{-1}$.
Regarding the term $\text{C}_{\task,s}$, exploiting the definition of the 
algorithm in \cref{RERM_bias}, we can write
$\text{C}_{\task,s}
\le \frac{\la}{2} ~ \| w_\task - \tau(s) \|^2$.
The desired statement follows by combining the two bounds 
above and optimizing w.r.t. $\la$.
\end{proof}
\cref{bound_fixed_bias} suggests that a conditioning function $\tau$ with low variance can potentially incur a small excess risk. 
This makes the minimizer of the variance, a potentially good candidate for conditional meta-learning. We note that ${\rm Var}_\env(\tau)$ in \cref{oracle} can be interpreted as a Least-Squares risk associated to the input-(ideal) output pair $(s,w_\task)$.  
Thanks to this interpretation, we can rely on the following well-known facts, 
see e.g. \cite[Lemma A$2$]{ciliberto2020general}. 
\begin{restatable}[Best conditioning function in hindsight]{lemma}{OracleFunction} \label{oracle_function}
The minimizer of ${\rm Var}_\rho(\cdot)^2$ in \cref{oracle} over the set $\T$ is such that $\tau_\env(s) = \Exp_{\task\sim\rho(\cdot|s)} ~ w_\mu$ almost everywhere on $\Ss$. Moreover, for any $\tau \in \T$,
\begin{equation} \label{gap_closed_form}
{\rm Var}_\env(\tau)^2 - {\rm Var}_\env(\tau_\env)^2
= \Exp_{s \sim \ms} ~ \big \| 
\tau(s) - \tau_\env(s) \big \|^2.
\end{equation}
\end{restatable}
Combining \cref{bound_fixed_bias} with \cref{oracle_function},
we can formally analyze when the conditional approach is significantly 
advantageous w.r.t. the unconditional one. \\

\paragraph{Conditional vs unconditional meta-learning}
As observed in \cref{conditional_meta_learning_problem_1}, unconditional meta-learning 
consists in restricting to the class of constant conditioning functions $\T^{\rm const}$. 
Minimizing ${\rm Var}_\rho(\cdot)^2$ over this class yields the optimal bias vector for 
standard meta-learning (see e.g. \cite{denevi2019learning, denevi2019online,balcan2019provable,khodak2019adaptive}),
given by the expected target tasks' vector
$w_\env = \Exp_{\task \sim \mt} ~ w_\task$.
Applying \cref{gap_closed_form} to the constant function
$\tau \equiv w_\env$, we get the following gap between the best 
performance of conditional and unconditional meta-learning:
\begin{equation} \label{gap_conditional_unconditional}
{\rm Var}_\env(w_\env)^2 - {\rm Var}_\env(\tau_\env)^2
~=~ \Exp_{s \sim \ms} ~ \nor{ ~ 
 w_\env - \tau_\env(s) ~ }^2.
\end{equation}
We note that the gap \cref{gap_conditional_unconditional} above is 
large when the ideal conditioning function $\tau_\env$ is ``far'' from 
being the constant function $w_\env$. We report below two 
examples that can be considered illustrative for many real-world scenarios 
in which such a condition is satisfied. We refer to \cref{examples} for the 
details and the deduction. 
In the examples, we parametrize each task with the triplet $\task=(w_\task,\eta_\task,\xi_\task)$,
where $w_\task$ is the target weight vector, $\eta_\task$ is the marginal distribution on the inputs, $\xi_\task$ is a noise model and $y\sim\mu(\cdot|x)$ is $y = \scal{w_\mu}{x} + \epsilon$ with $x\sim\eta_\task$ and $\epsilon\sim\xi_\task$. Additionally, we denote by $\mathcal{N}(v,\sigma^2 I)$ a Gaussian distribution with mean $v\in\R^d$ and covariance matrix $\sigma^2 I$, with $I$ the $d \times d$ identity matrix. 
\begin{example}[Clusters of tasks] \label{clusters_ex}
Let $\env_\Tt = \frac{1}{m} \sum_{i=1}^m \env_\Tt^{(m)}$ be a uniform mixture of $m$ environments (clusters) of tasks. For each $i=1,\dots,m$, a task $\task \sim\env_\Tt^{(i)}$ is sampled such that: $1)$ $w_\mu\sim\mathcal{N}(w(i),\sigma_w^2I)$ with $w(i) \in\R^d$ a cluster's mean vector and $\sigma_w^2I$ a covariance matrix, $2)$ the marginal $\eta_\task = \mathcal{N}(x(i),\sigma_\X^2)$ with mean vector $x(i)\in\R^d$ and variance $\sigma_\X^2$, $3)$ the side information is an $n$ i.i.d. sample from $\eta_\task$, namely $s = (x_i)_{i=1}^n \sim \eta_\task^n$. Then, the gap between conditional and unconditional variance is 
\begin{equation}\label{eq:lower-bound-clusters}
{\rm Var}_\env(w_\env)^2 - {\rm Var}_\env(\tau_\env)^2 
~\geq~ \frac{1}{2 m ^2} \sum_{i,j = 1}^m 
\Bigg(1 - \frac{m}{2} ~ e^{- \frac{n}{\sigma_\X^2}
\nor{x(i) - x(j)}^2} \Bigg) \nor{w(i)-w(j)}^2.
\end{equation}
\end{example}
The inequality above confirms our natural intuition. It tells us 
that the larger is the number of clusters and the more the target 
weight vectors' and inputs' centroids are distant (i.e. the more the 
clusters are distant and the inputs' side information are discriminative 
for conditioning), the more the conditional approach will be 
advantageous w.r.t. the unconditional one.

\begin{example}[Curve of tasks] \label{curve_ex}
Let $\env_\Ss$ be a uniform distribution over $\Ss = [0,1]$. Let $h:\Ss\to\R^d$ parametrize a circle of radius $r>0$ centered in $c\in\R^d$, such as $h(s) = r~(\cos(2\pi s),\sin(2\pi s),0,\dots,0)^\top$. For $s\in\Ss$, let $\task\sim\env(\cdot|s)$ such that $w_\task\sim\mathcal{N}(h(s),\sigma^2I)$ with $\sigma \in \Real$. Then, $\tau_\rho = h$, $w_\rho = c$ and the the gap between conditional and unconditional variance is 
\begin{equation}\label{eq:gap-circle}
{\rm Var}_\env(w_\env)^2 - {\rm Var}_\env(\tau_\env)^2 ~=~ r^2.
\end{equation}
\end{example}
Hence, in this case, the advantage in applying the conditional approach 
w.r.t. the unconditional one is equivalent to the squared radius of the circle 
over which the mean of the target weight vectors $w_\task$ lie. \\

\paragraph{Conditional meta-learning vs Independent Task Learning (ITL)}
Solving each task independently corresponds to choosing
the constant conditioning function $\tau_0 \equiv 0$. 
Applying \cref{oracle_function} to this function, the gap between 
the performance of the best conditional approach and ITL reads as
\begin{equation} \label{cond_vs_ITL}
{\rm Var}_\env(0)^2 - {\rm Var}_\env(\tau_\env)^2 
~ = ~\Exp_{s \sim \ms} ~ \big \| w_\env - \tau_\env(s) \big \|^2 
~+~ \nor{w_\env}^2.
\end{equation}
The gap in \cref{cond_vs_ITL} combines the gain of conditional 
over unconditional meta-learning with $\nor{w_\env}^2 
= {\rm Var}_\env(0)^2 - {\rm Var}_\env(w_\env)^2$ that is 
the advantage of unconditional meta-learning over ITL (see  
\cite{denevi2019learning, denevi2019online}). 
In the next section, we introduce a convex meta-algorithm mimicking this 
advantage also in practice.


\section{Conditional meta-learning algorithm}
\label{proposed_method}

To address conditional meta-learning in practice, we introduce the following set of conditioning functions. For a given feature map $\Phi:\Ss\to\R^k$ on the side 
information space, we define the associated space of linear functions
\begin{equation}\label{eq:feature-space-T}
\T_\Phi = \Big \{ \tau: \Ss \to \Real^d ~ \big | ~ \tau(\cdot) 
= M \Phi(\cdot) + b, \text{ for some } M \in \Real^{d \times k}, b \in \Real^d \Big \}.
\end{equation}
To highlight the dependency of a function $\tau \in \T_\Phi$ w.r.t. its parameters $M$ 
and $b$, we will use the notation $\tau= \tau_{M,b}$. 
Evidently, $\T_\Phi$ contains the space of all unconditional estimators 
$\T^{\rm const}$. We consider $\T_\Phi$ equipped with the 
canonical norm $\nor{\tau_{M,b}}^2 = \nor{(M,b)}_F^2 = \nor{M}_F^2 + \nor{b}^2$, with $\| \cdot \|_F$ the Frobenius norm. We now introduce two 
standard assumptions will allow the design of our method.

\begin{assumption} \label{ass_2}
The minimizer $\tau_\rho$ of ${\rm Var}_\rho(\cdot)$ belongs to $\T_\Phi$, namely there exist $M_\rho\in\R^{d \times k}$ and $b_\rho\in\R^d$, such that $\tau_\rho(\cdot) = M_\rho\Phi(\cdot) + b_\rho$.
\end{assumption}

\begin{assumption} \label{ass_3}
There exists $K>0$ such that $\| \Phi(s) \| \le K$ for any $s\in\Ss$. 
\end{assumption} 
\cref{ass_2} enables us to restrict the conditional meta-learning problem
in \cref{conditional_meta_learning_problem_1} to $\T_\Phi$, rather than 
to the entire space $\T$ of measurable functions. 
In \cref{oracle_meta_parameters} in \cref{proof_oracle_meta_parameters}
we provide the closed forms of $M_\env$ and $b_\env$ and we express the 
gap in \cref{gap_conditional_unconditional} by the correlation between 
$w_\task$ and $\Phi(s)$ and the slope of $\tau_\env$.
\cref{ass_3} will allow us to work with a Lipschitz meta-objective,
as explained below. \\

\paragraph{The convex surrogate problem}
Following a similar strategy to the one adopted for the unconditional setting 
in \cite{denevi2019learning, denevi2019online}, we introduce the 
following surrogate problem for the conditional one in \cref{conditional_meta_learning_problem_1}:
\begin{equation} \label{surrogate}
\min_{\tau \in \T} ~ \hat \ee_\env(\tau)
\quad \quad \quad 
\hat \ee_\env(\tau) ~=~ \Exp_{(\task, s) \sim \env} ~ \Exp_{\Zn \sim \task^n}
~~ \cR_\Zn^\la(A(\tau(s),Z)),
\end{equation}
where we have replaced the inner expected risk $\cR_\task$ with the regularized empirical risk $\cR_Z^\la$ in \cref{RERM_bias}. 
Exploiting \cref{ass_2}, the problem above can be rewritten more explicitly as follows
\begin{equation} \label{surrogate_linear}
\min_{M\in \Real^{d \times k}, b \in \Real^d} ~ \Exp_{(\task, s) \sim \env}
~ \Exp_{\Zn \sim \task^n} ~ \LL\big (M, b, s, \Zn \big )
\qquad \LL\big (M, b, s, \Zn \big ) = \cR_Z^\la(A(\tau_{M,b}(s),Z) ).
\end{equation}
The following proposition characterizes useful properties of the meta-loss 
$\LL\big (\cdot, \cdot, s, \Zn \big )$ introduced above (such as convexity 
and differentiability) and it supports its choice as surrogate meta-loss. We 
denote by $\cdot \trans$ the standard transposition operation. 

\begin{restatable}[Properties of the surrogate meta-loss $\LL$]{proposition}{PropertiesSurrogate} 
\label{properties_surrogate}
For any $\Zn \in \D$ and $s \in \Ss$, the function 
$\LL\big (\cdot, \cdot, s, \Zn \big )$ 
is convex, differentiable and its gradient is given by
\begin{equation} \label{gradient1}
\nabla \LL\big (\cdot, \cdot, s, \Zn \big )(M, b) =
- \la \Big( A \big(\tau_{M, b}(s), \Zn \big) - 
\tau_{M, b}(s) \Big) \left(\begin{array}{c}\Phi(s)\\ 1\end{array}\right) \trans
\end{equation}
for any $M \in\R^{d\times k}$ and $b \in \Real^d$. 
Moreover, under \cref{ass_1} and \cref{ass_3}, we have
\begin{equation} \label{bound_grad_norm}
\nor{\nabla \LL\big (\cdot, \cdot, s, \Zn \big )(M, b)}_F^2
\le L^2 \rx^2 (K^2 + 1).
\end{equation}
\end{restatable}
The proof of \cref{properties_surrogate} is reported in 
\cref{properties_surrogate_proof} and it follows a similar 
reasoning in \cite{denevi2019online}, by taking into account 
also the parameter $M$ in the optimization problem. \\

\paragraph{The conditional meta-learning estimator} 
In this work we propose to apply  Stochastic Gradient Descent (SGD) on the surrogate problem in \cref{surrogate_linear}. \cref{OGDA2_paper} summarizes the implementation of this approach: assuming a sequence of i.i.d. pairs $(\Zn_t, s_t)_{t = 1}^T$ of training sets and side information, at each iteration the algorithm updates the conditional iterates $(M_t,b_t)$ by performing a step of constant size $\gamma>0$ in the direction of $-\nabla\LL(\cdot,\cdot,s_t,\Zn_t)(M_t,b_t)$. The map $\tau_{\thickbar M, \thickbar b}$ is then returned as conditional estimator, with $(\thickbar M,\thickbar b)$ the average across all the iterates $(M_t,b_t)_{t = 1}^T$. The following result characterizes the excess risk of the proposed estimator. 

\makeatletter
\algrenewcommand\ALG@beginalgorithmic{\footnotesize}
\makeatother
\begin{algorithm}[t]
\caption{Meta-Algorithm, SGD on \cref{surrogate_linear}}\label{OGDA2_paper}
\begin{algorithmic}
\State ~
   \State {\bfseries Input} ~ $\gamma > 0$ meta-step size, $\la > 0$ inner regularization parameter
   \vspace{.075cm}
   \State {\bfseries Initialization} ~ $M_1 = 0 \in \Real^{d \times k}$, $b_1 = 0 \in \Real^d$
  \vspace{.075cm}
   \State {\bfseries For} ~ $t=1$ to $T$
   \vspace{.075cm}
   \State \qquad ~ Receive ~~ $(\task_t, s_t) \sim \env$ and $\Zn_t \sim \task_t^n$
   \vspace{.075cm}
   \State \qquad ~ Let ~~ $\theta_t ~=~ \tau_{M_t, b_t}(s_t) ~=~ M_t \Phi(s_t) + b_t$
   \vspace{.075cm}
   \State \qquad ~ Run the inner algorithm in \cref{RERM_bias} to obtain $w_t ~=~ A(\theta_t, \Zn_t)$
   \vspace{.05cm}
   \State \qquad ~ Compute $\nabla \LL (\cdot, \cdot, s_t, \Zn_t  )(M_t, b_t) = -\lambda (w_t - \theta_t) \Big(\begin{array}{c}\Phi(s_t)\\ 1\end{array}\Big) \trans$ as in \cref{gradient1}
   \vspace{.05cm}
   \State \qquad ~ Update ~~ $(M_{t+1}, b_{t+1}) = (M_t, b_t) - \gamma \nabla \LL (\cdot, \cdot, s_t, \Zn_t)(M_t, b_t)$
   \vspace{.075cm}
 \State {{\bfseries Return} ~ $\displaystyle {\thickbar M} = \frac{1}{T} \sum_{t=1}^T M_t$, $\displaystyle \thickbar b = \frac{1}{T} \sum_{t=1}^T b_t$} 
\State ~
\end{algorithmic}
\end{algorithm}

\begin{restatable}[Excess risk bound for the conditioning function 
returned by \cref{OGDA2_paper}]{theorem}{BoundEstimatedBias}\label{bound_estimated_bias}
Let \cref{ass_1} and \cref{ass_3} hold. 
Let $\tau_{M,b}$ be a fixed function in $\T_\Phi$ and let
${\rm Var}_\env(\tau_{M,b})^2$ be the corresponding variance 
introduced in \cref{conditional_var}. Let $\thickbar M$ and 
$\thickbar b$ be the outputs of \cref{OGDA2_paper} applied to a 
sequence $(\Zn_t, s_t)_{t = 1}^T$ of i.i.d. pairs sampled from 
$\env$ with inner regularization parameter and meta-step size
\begin{equation}
\la ~=~ \frac{2 \rx L}{{\rm Var}_\env(\tau_{M,b})}~ \frac{1}{\sqrt{n}}
\quad \quad \quad 
\gamma ~=~ \frac{\nor{(M,b)}_F}{L \rx \sqrt{(K^2 + 1)}} ~ \frac{1}{\sqrt{T}}.
\end{equation}
Then, in expectation w.r.t. the sampling of $(\Zn_t, s_t)_{t = 1}^T$, 
\begin{equation} \label{conditional_bound_bias}
\Exp ~ \ee_\env(\tau_{\thickbar M,\thickbar b}) - \ee_\env^*
~\leq~ \frac{2 \rx L {\rm Var}_\env(\tau_{M,b})}{\sqrt{n}} 
+ \frac{L \rx \sqrt{K^2 + 1} \nor{(M,b)}_F}{\sqrt{T}}.
\end{equation}
\end{restatable}

\begin{proof}[Proof (Sketch)]
We consider the following decomposition
\begin{equation} \label{decomposition_bias}
\Exp ~ \ee_\env(\tau_{\thickbar M, \thickbar b}) - \ee_\env^* =
\underbrace{\Exp ~ \ee_\env(\tau_{\thickbar M, \thickbar b}) - 
\hat \ee_\env(\tau_{\thickbar M, \thickbar b})}_{\text{B}}
+ \underbrace{\Exp ~ \hat \ee_\env(\tau_{\thickbar M, \thickbar b}) 
- \hat \ee_\env(\tau_{M, b})}_{\text{C}}
+ \underbrace{\hat \ee_\env(\tau_{M, b}) - \ee_\env^*}_{\text{D}}.
\end{equation}
Applying \cref{ass_1} and the stability arguments in 
\cref{generalization_error_RERM} in \cref{generalization_error_RERM_sec}, 
we can write $\text{B} \le 2 \rx^2 L^2 (\la n)^{-1}$.
The term C is the term expressing the convergence rate of 
\cref{OGDA2_paper} on the surrogate problem in \cref{surrogate_linear} 
and, exploiting \cref{ass_3} and \cref{properties_surrogate}, it can be 
controlled as described in \cref{convergence_surrogate} in 
\cref{proof_conv_rate_surr}. 
Regarding the term D, exploiting the definition of the 
algorithm in \cref{RERM_bias}, we can write $\text{D}
\le \frac{\la}{2} ~ {\rm Var}_\env(\tau_{M,b})^2$.
Combining all the terms and optimizing w.r.t. $\gamma$ and $\la$, 
we get the desired statement.
\end{proof}
We now comment about the result we got above in \cref{bound_estimated_bias}. \\

\paragraph{Proposed vs optimal conditioning function}
Specializing the bound in \cref{bound_estimated_bias}
to the best conditioning function $\tau_\env$ in 
\cref{oracle_function}, thanks to \cref{ass_2}, 
we get the following bound for our estimator:
\begin{equation}
\Exp ~ \ee_\env(\tau_{\thickbar M,\thickbar b}) - \ee_\env^* ~\leq~
\mathcal{O} \Big( {\rm Var}_\env(\tau_\env) ~ n^{-1/2} 
+ \|(M_\env,b_\env) \|_F ~ T^{-1/2}\Big).
\end{equation}
Hence, our proposed meta-algorithm achieves comparable 
performance to the best conditioning function $\tau_\env$ in hindsight, 
provided that the number of observed tasks is sufficiently large.
The bound above also highlights the trade-off between statistical 
and computational complexity of the class $\T_\phi$: conditional meta-learning 
incurs in a cost $\|(M_\env,b_\env) \|_F$ in the $\sqrt{T}$-term that is larger 
than the $\nor{b_\env}$ cost of unconditional meta-learning
(see \cite{denevi2019learning,balcan2019provable,khodak2019adaptive}), 
which is, however, limited 
to constant conditioning functions. This is an acceptable price, 
since, as we discussed in \cref{unconditional_for_bias_advantage}, 
the performance of conditional meta-learning is significantly 
better than the standard one in many common scenarios.
\begin{remark}
When $\tau_\env \notin \T_\Phi$ (i.e. when \cref{ass_3} does not hold), 
our method suffers an additional approximation error due to the fact 
$\min_{\tau \in \T_\Phi} ~ {\rm Var}_\env(\tau) > {\rm Var}_\env(\tau_\env)$.
In this case, one might nullify the gap above by considering a feature 
map $\Phi:\Ss\to\hh$ with $\hh$ a universal reproducing kernel Hilbert 
space of functions. Exploiting standard arguments from online learning 
with kernels literature (see e.g. \cite{kivinen2004online,singh2012online,shalev2014understanding}),
in \cref{lemma_only_kernel} in \cref{implementation_kernels} we 
describe the implementation of \cref{OGDA2_paper} for this setting 
using only evaluations of the kernel associated to the feature map.
We leave the corresponding theoretical analysis to future work.
\end{remark}

\paragraph{Proposed conditioning function vs unconditional meta-learning}
Specializing \cref{bound_estimated_bias}
to $\tau_{M,b} \equiv w_\env$, the bound for our estimator becomes:
\begin{equation} \label{unconditional_retrieved}
\Exp ~ \ee_\env(\tau_{\thickbar M,\thickbar b}) - \ee_\env^* ~\leq~
\mathcal{O} \big( {\rm Var}_\env(w_\env) ~ n^{-1/2}
+ \| w_\env \| ~ T^{-1/2}\big),
\end{equation}
which is equivalent to state-of-the-art bounds for 
unconditional methods, see \cite{denevi2019learning, denevi2019online,balcan2019provable,khodak2019adaptive}.
Hence, our conditional approach provides, at least,
the same guarantees as its unconditional counterpart. \\

\paragraph{Proposed conditioning function vs ITL}
Specializing \cref{bound_estimated_bias} to 
$\tau_{M,b} \equiv 0$ corresponds to force $\gamma = 0$ and, 
consequently, \cref{OGDA2_paper} to not move. In such 
a case, we get the bound:
\begin{equation} \label{comparison_ITL}
\Exp ~ \ee_\env(\tau_{\thickbar M,\thickbar b}) - \ee_\env^* ~\leq~ 
\mathcal{O} \big( {\rm Var}_\env(0) ~ n^{-1/2} \big),
\end{equation}
which corresponds to the standard excess risk bound for ITL, 
see \cite{denevi2019learning, denevi2019online,balcan2019provable,khodak2019adaptive}.
In other words, our method does not generate negative transfer effect.
\begin{remark}[Fine-tuning] \label{approx_meta_sub}
In the case of the online inner family in \cref{online_algorithm_remark}
used in fine-tuning, \cref{OGDA2_paper} employs an approximation of the 
meta-subgradient in \cref{gradient1} by replacing the batch regularized 
empirical risk minimizer $A (\tau_{M, b}(s), \Zn )$ in \cref{RERM_bias} 
with the last iterate of the online algorithm in \cref{online_inner_algorithm}.
As shown in \cite{denevi2019learning, denevi2019online}
for the unconditional setting, such an approximation does not affect the 
behavior of the bounds above. 
\end{remark}


\section{Experiments}
\label{experiments}

In this section we compare the numerical performance 
of our conditional method in \cref{OGDA2_paper} (cond.) w.r.t. its unconditional 
counterpart in \cite{denevi2019learning} (uncond.). We will also add to the 
comparison the methods consisting in applying the inner algorithm on each 
task with $\tau \equiv 0 \in \Real^d$ (i.e. ITL) and the unconditional 
oracle $\tau \equiv w_\env = \Exp_{\task \sim \mt} ~ w_\task$ (mean), 
when available. We considered regression problems and we evaluated the errors 
by the absolute loss. The results refer to the fine-tuning 
variant of the methods with the online inner algorithm in \cref{online_inner_algorithm}. 
For all the experiments below (except the synthetic circle),
we used as side information collections of datapoints 
(see \cref{double_sample_size}). \\

\paragraph{Synthetic clusters.} We considered three variants of the 
setting described in \cref{clusters_ex}. In all the variants we sampled 
$T_{\rm tot} = 480$ tasks from a mixture of $m$ clusters with the same 
probability. For each task $\task$, we sampled the corresponding target 
vector $w_\task$ from the $d = 20$-dimensional Gaussian distribution 
$\mathcal{N}(w(j_\task),I)$, where, $j_\task \in \{1, \dots, m\}$ denotes 
the cluster from which the task $\task$ was sampled. 
We then generated the corresponding dataset $(x_i,y_i)_{i=1}^{n_{\rm tot}}$ 
with $n_{\rm tot} = 20$. We sampled the inputs from 
$\mathcal{N}(x(j_\task), I)$ and we generated the labels according 
to the equation $y = \langle x, w_\task \rangle + \epsilon$, with the noise 
$\epsilon$ sampled from $\mathcal{N}(0, \sigma^2 I)$, with $\sigma$
chosen in order to have signal-to-noise ratio equal to $1$. \\
In \cref{fig_exps} (left-top), we generated an environment as above with 
just one cluster ($m = 1$) and we took $w(1) = 4 \in \Real^d$ (the vector in 
$\Real^d$ with all components $4$) and $x(1) = 1 \in \Real^d$. As we can 
see, coherently with previous work \cite{denevi2019learning}, the uncoditional 
approach outperforms ITL and it converges to the mean vector $w_\env = w(1)$ 
as the number of training tasks increases. The conditional approach returns 
equivalent performances to the unconditional counterpart. \\
In \cref{fig_exps} (right-top), we considered an environment of two clusters 
($m = 2$) identified by $w(1) = 8 \in\Real^d$, $w(2) = 0 \in \Real^d$ 
(implying $w_\env = 4$), $x(1) = 1 \in \Real^d$ and $x(2) = - x(1)$. 
As we can see, the conditional approach outperform ITL 
as in the previous setting, but the conditional approach yields even 
better performance. \\
Finally, in \cref{fig_exps} (left-bottom), we considered 
an environment of two clusters ($m = 2$) identified by $w(1) = 4 \in \Real^d$, 
$w(2) = - w(1)$ (implying $w_\env = 0$), $x(1) = 1 \in \Real^d$ and 
$x(2) = - x(1)$. As expected, the uncoditional approach 
mimics the poor performance of ITL, while, the performance of the 
conditional approach is promising. \\
Summarizing, the conditional approach 
brings advantage w.r.t. the unconditional one when the heterogeneity 
of the environment is significant. When the environment is homogeneous, 
the performance of the two are equivalent. This conclusion is exactly inline
with our theory in \cref{unconditional_retrieved} and \cref{comparison_ITL}. \\

\begin{figure}[t]
\begin{minipage}[t]{0.49\textwidth}  
\centering
\includegraphics[width=1\textwidth]{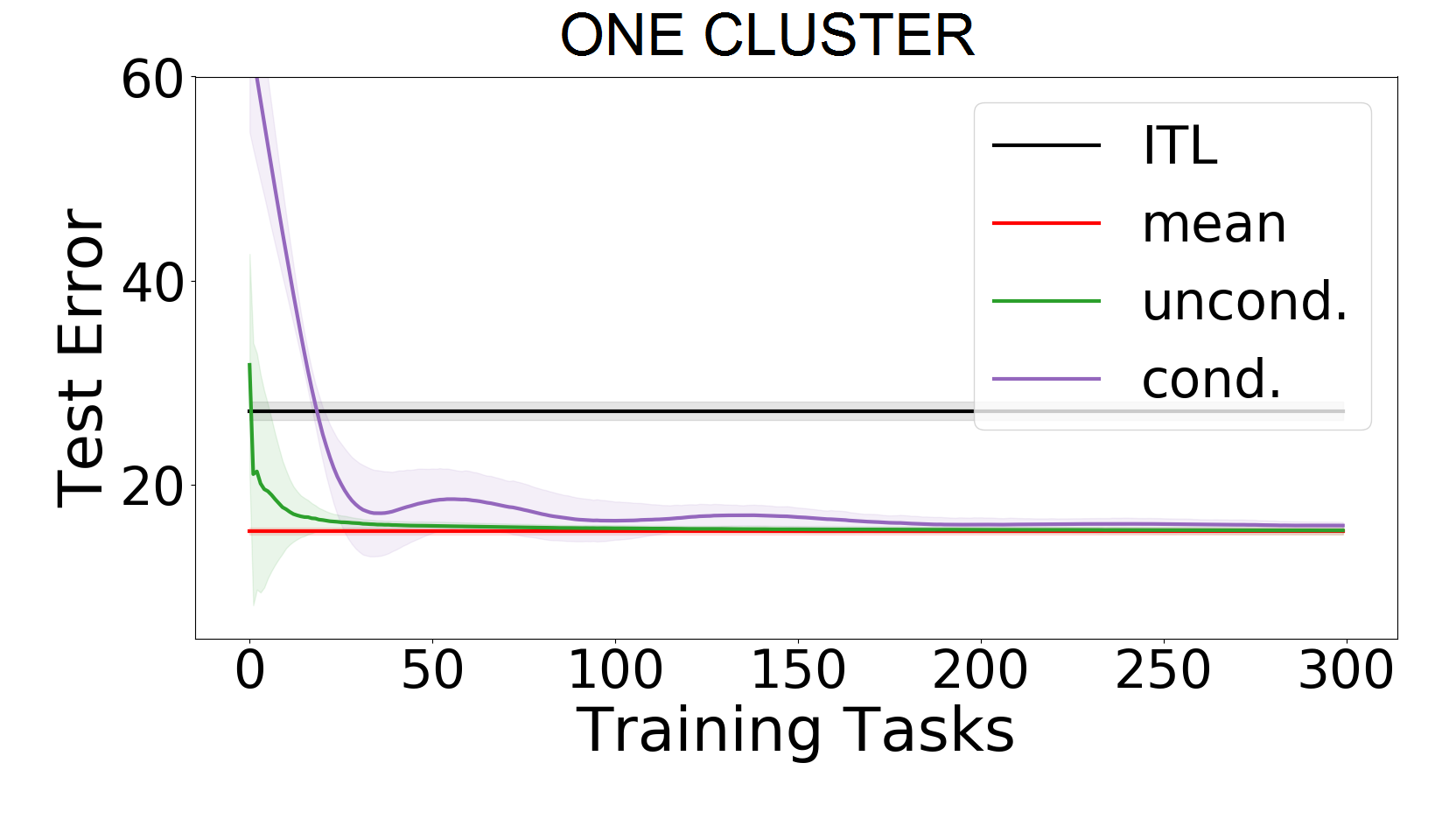} \\ 
\includegraphics[width=1\textwidth]{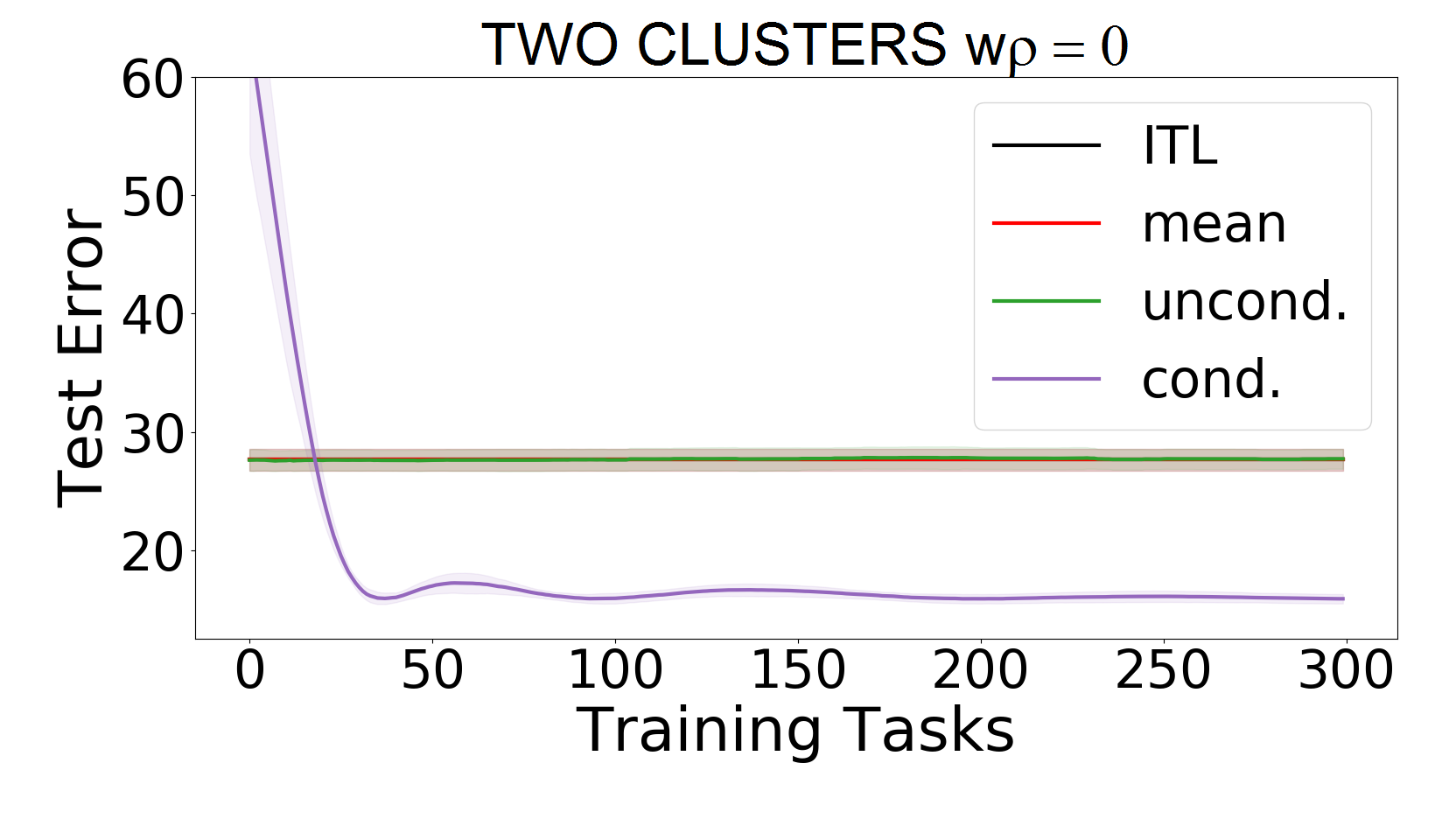}
\end{minipage}
\begin{minipage}[t]{0.49\textwidth}
\centering
\includegraphics[width=1\textwidth]{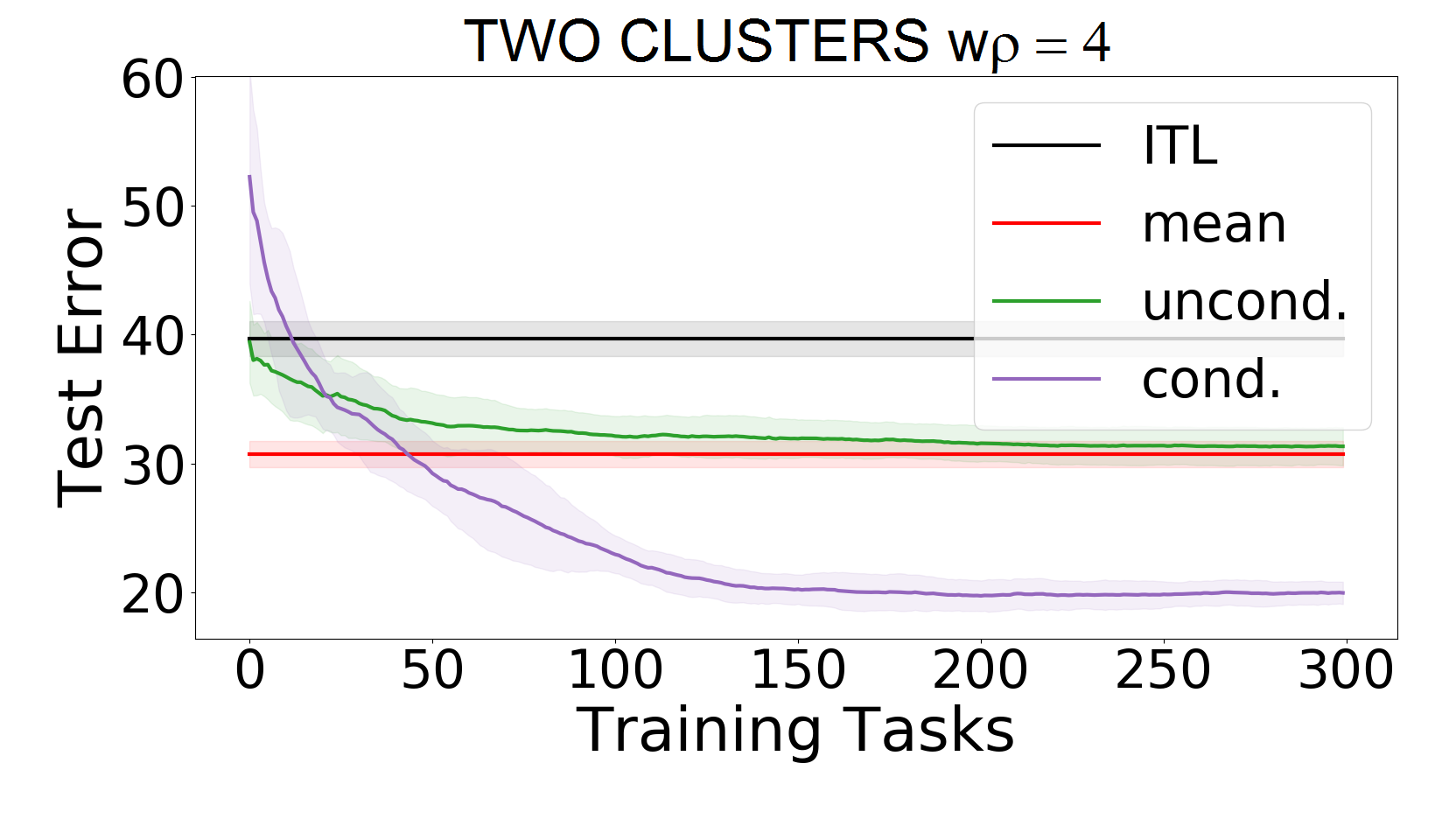} \\
\includegraphics[width=1\textwidth]{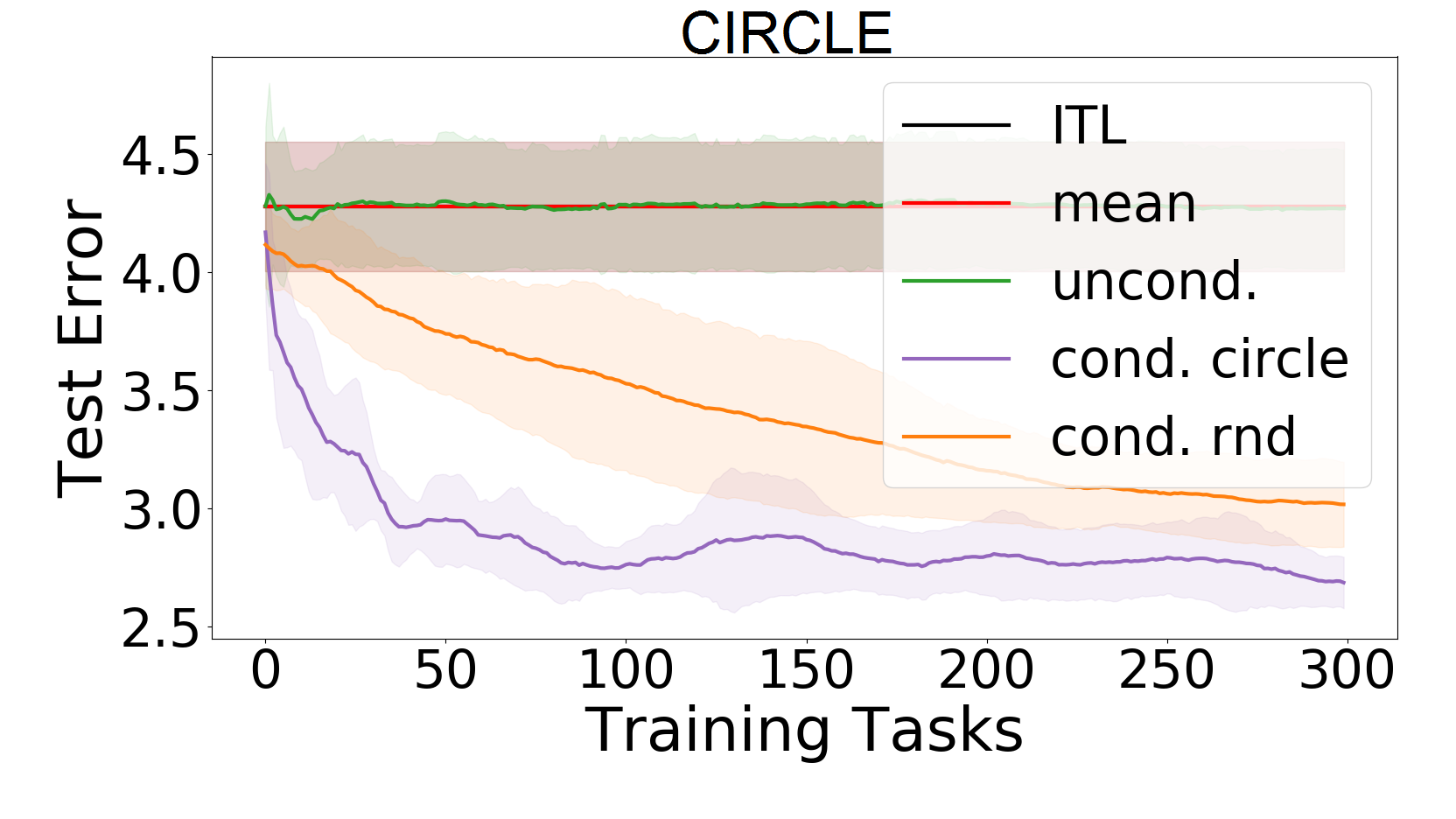} \\ 
\end{minipage}
\caption{Performance (averaged over $10$ seeds) 
of different methods w.r.t. an increasing number of tasks
on different environments: with one cluster
(left-top), with two clusters and $w_\env = 4$ 
(right-top), with two clusters and $w_\env = 0$ 
(left-bottom), circle (right-bottom).
\label{fig_exps}}
\vspace{-.3cm}
\end{figure}

\paragraph{Synthetic circle.} We sampled $T_{\rm tot} = 480$ tasks
according to the setting described in \cref{curve_ex}. Specifically, 
for each task $\task$, we first sampled the corresponding side 
information $s \in [0,1]$ according to the uniform distribution. 
We then generated the vector 
\begin{equation}
h(s) = r~(\cos(2\pi s),\sin(2\pi s),
0,\dots,0)^\top \in \Real^d,
\end{equation}
with $d = 20$, on the zero-centered 
circle of radius $r = 8$. After this, we sampled the 
corresponding target weight vector $w_\task$ from $\mathcal{N}(h(s), I)$.
We then generated the associated dataset of $n_{\rm tot} = 20$ points 
as for the experiments above. We applied our conditional approach with 
the true underlying feature map $\Phi(s) = 
({\rm cos}(2 \pi s), {\rm sin}(2 \pi s))$ (cond. circle) and a feature map 
mimicking a Gaussian distribution by Fourier random features
\cite{rahimi2008random} (cond. rnd). \\
From \cref{fig_exps} (right-bottom) we see that the performance of 
unconditional meta-learning mimics the poor performance of ITL 
(in fact, we have $w_\env = 0$). On the other hand, both the conditional 
approaches bring a substantial advantage and the random features'
variant approaches the variant knowing the true underlying feature map.

Because of lack of space, in \cref{exp_details}, we report two additional 
experiments showing the effectiveness of our conditional approach on
two real datasets (the Lenk \cite{lenk1996hierarchical,Andrew} and the 
Schools \cite{argyriou2008convex} datasets). We also describe 
the implementation details we omit here, such as the feature map 
$\Phi$ and the hyper-parameters $\la, \gamma$ we used.


\section{Conclusion}
\label{conclusion}

We proposed a new conditional meta-learning framework for biased regularization
and fine-tuning based on side information and we provided a theoretical analysis 
demonstrating its potential advantage over standard meta-learning, when the 
environment of tasks is heterogeneous. In the future, taking inspiration from 
\cite{orabona2016coin,cutkosky2018black}, it would be interesting to develop 
a variant of our method in which the hyper-parameters are automatically tuned 
in efficient way. In addition, it would valuable to extend our conditional approach 
and the corresponding analysis to other meta-learning paradigms considering different 
families of inner algorithms, such as \cite{tripuraneni2020provable,denevi2019online}.


\section*{Acknowledgments}

This work was supported in part by SAP SE and EPSRC Grant N. EP/P009069/1. C.C. acknowledges the Royal Society (grant SPREM RGS\textbackslash R1\textbackslash 201149). 



{\small
\bibliographystyle{abbrv}
\bibliography{references}
}



\newpage

\appendix

\section*{Appendix}

The supplementary material is organized as follows. 
In \cref{generalization_error_RERM_sec} we give the bound on the 
generalization error of the algorithm in \cref{RERM_bias} that 
we used in various proofs. In \cref{examples} we formally describe 
the deduction of the statements reported in \cref{clusters_ex}
and \cref{curve_ex} in \cref{unconditional_for_bias_advantage}.
In \cref{proof_oracle_meta_parameters} we report the closed 
form of $M_\env \in \Real^{d \times k}$ and $b_\env \in \Real^d$
in \cref{ass_2} and we express the gap between the conditional 
and the uncoditional variance in \cref{gap_conditional_unconditional} 
by the correlation between the target tasks' vectors $w_\task$ and 
the transformed side information $\Phi(s)$ or the slope of $\tau_\env$.
In \cref{proofs_proposed_method_sec}, 
we report the proofs of the statements we used in \cref{proposed_method} 
in order to prove the expected excess risk bound in \cref{bound_estimated_bias}
for \cref{OGDA2_paper}. Finally, in \cref{exp_details}, we report
two additional real experiments and the implementation details we omitted in
the main body, because of lack of space.


\section{Generalization error of the algorithm in \cref{RERM_bias}}
\label{generalization_error_RERM_sec}

In this section we report the generalization error bound of the family
of inner algorithms in \cref{RERM_bias} that we used in our proofs. The
statement exploits standard tools from stability theory. We do not claim
any originality, we report the proof for completeness.

\begin{restatable}[Generalization error of the algorithm in \cref{RERM_bias}]{proposition}{generalizationErrorRERM} 
\label{generalization_error_RERM}
For a distribution $\task \sim \env$, fix a dataset $\Zn = (x_i,y_i)_{i = 1}^n \sim \task^n$ and, for any $i \in \{1, \dots, n \}$, fix a datapoint $z_i' = (x_i', y_i') \sim \task$ independent from $\Zn$. For any $\theta \in \Theta$ not depending on $\Zn$, let $\hat w_\theta(\Zn) = A(\theta, \Zn)$ be the output of the algorithm in \cref{RERM_bias} over $\Zn$ and let $s_{\theta,i}' \in \partial \ell ( \cdot , y_i' ) (  \langle x_i', \hat w_\theta(\Zn) \rangle)$ be a subgradient of $\ell ( \cdot , y_i' )$ at $\langle x_i', \hat w_\theta(\Zn) \rangle$. Then, the following generalization error bound holds for $\hat w_\theta(\Zn)$
\begin{equation}
\Exp_{\Zn \sim \task^n} ~ \big[ \cR_\task ( \hat w_\theta(\Zn)) - \cR_{\Zn} ( \hat w_\theta(\Zn)) \big] \le \frac{2}{\la n} ~ \Exp_{\Zn \sim \task^n} ~ \Exp_{z_i' \sim \task} ~ \big \| x_i' s_{\theta,i}' \big \|^2.
\end{equation}
As a consequence, under \cref{ass_1}, the right side term 
above can be upper bounded by $2 L^2 \rx^2 (\la n)^{-1}$.
\end{restatable}

\begin{proof}
For any $i \in \{ 1, \dots, n \}$, consider the dataset $\Zn^{(i)}$, a copy of the
original dataset $\Zn$ in which we exchange the point $z_i = (x_i, y_i)$ with the new
i.i.d. point $z_i' = (x_i', y_i')$. For a fixed $\theta \in \Theta$, we analyze how much this
perturbation affects the outputs of the algorithm in \cref{RERM_bias}. In other words, we 
study the discrepancy between $\hat w_\theta(\Zn)$ and $\hat w_\theta(\Zn^{(i)})$. 
We start from observing that, since $\cR^\la_{\Zn}$ is $\la$-strongly convex w.r.t. 
$\| \cdot \|$, by growth condition and the definition of the algorithm in \cref{RERM_bias}, 
we can write the following
\begin{equation}
\begin{split}
& \frac{\la}{2} ~ \big \| \hat w_\theta(\Zn^{(i)}) - \hat w_\theta(\Zn) \big \|^2  \le 
\cR^\la_{\Zn} ( \hat w_\theta(\Zn^{(i)}) )  - \cR^\la_{\Zn} ( \hat w_\theta(\Zn) ) \\
& \frac{\la}{2} ~ \big \| \hat w_\theta(\Zn^{(i)}) - \hat w_\theta(\Zn) \big \|^2  \le 
\cR^\la_{\Zn^{(i)}} ( \hat w_\theta(\Zn) )  - \cR^\la_{\Zn^{(i)}} 
( \hat w_\theta(\Zn^{(i)}) ).
\end{split}
\end{equation}
Hence, summing the two inequalities above, we get
\begin{equation} \label{o}
\begin{split}
\la ~ \big \| \hat w_\theta(\Zn^{(i)}) - \hat w_\theta(\Zn) \big \|^2
& \le \cR^\la_{\Zn} ( \hat w_\theta(\Zn^{(i)}) )  
- \cR^\la_{\Zn^{(i)}} ( \hat w_\theta(\Zn^{(i)}) ) 
+ \cR^\la_{\Zn^{(i)}} ( \hat w_\theta(\Zn) ) - 
\cR^\la_{\Zn} ( \hat w_\theta(\Zn) ) \\
& = \frac{{\text B} + {\text C}}{n},
\end{split}
\end{equation}
where we have introduced the terms
\begin{equation}
\begin{split}
& \text{B} = \ell (  \langle x_i', \hat w_\theta(\Zn)  \rangle, y_i' ) - \ell (  \langle x_i', \hat w_\theta(\Zn^{(i)})  \rangle, y_i' ) \\
& \text{C} = \ell (  \langle x_i, \hat w_\theta(\Zn^{(i)})  \rangle, y_i ) - \ell (  \langle x_i, \hat w_\theta(\Zn)  \rangle, y_i ).
\end{split}
\end{equation}
Now, exploiting the assumption $s_{\theta,i}' \in \partial \ell ( \cdot , y_i' ) (  \langle x_i', \hat w_\theta(\Zn)  \rangle )$,
applying Holder's inequality and introducing a subgradient
$s_{\theta,i} \in \partial \ell ( \cdot , y_i ) (  \langle x_i, \hat w_\theta(\Zn^{(i)})  \rangle)$, we can write
\begin{equation} \label{oo}
\begin{split}
& \text B \le \big \langle x_i' s_{\theta,i}', \hat w_\theta(\Zn) - \hat w_\theta(\Zn^{(i)}) \big \rangle
\le \big \| x_i' s_{\theta,i}' \big \| ~ \big \| \hat w_\theta(\Zn^{(i)}) - \hat w_\theta(\Zn) \big \| \\
& \text C \le \big \langle x_i s_{\theta,i}, \hat w_\theta(\Zn^{(i)}) - \hat w_\theta(\Zn) \big \rangle
\le \big \| x_i s_{\theta,i} \big \| ~ \big \| \hat w_\theta(\Zn^{(i)}) - \hat w_\theta(\Zn) \big \|.
\end{split}
\end{equation}
Combining these last two inequalities with \cref{o} and simplifying, we get the following
\begin{equation} \label{oooo}
\big \| \hat w_\theta(\Zn^{(i)}) - \hat w_\theta(\Zn) \big \| \le \frac{1}{\la n} 
\Bigl( \big \| x_i' s_{\theta,i}' \big \| + \big \| x_i s_{\theta,i} \big \| \Bigr).
\end{equation}
Hence, combining the first row in \cref{oo} with \cref{oooo}, we can write
\begin{equation} \label{mela}
\begin{split}
\ell (  \langle x_i', \hat w_\theta(\Zn)  \rangle, y_i' ) - \ell ( \langle x_i', 
\hat w_\theta(\Zn^{(i)}) \rangle, y_i' )
\le \frac{1}{\la n} \Bigl( \big \| x_i' s_{\theta,i}' \big \|^2 
+ \big \| x_i' s_{\theta,i}' \big \| ~
\big \| x_i s_{\theta,i} \big \| \Bigr).
\end{split}
\end{equation}
Now, taking the expectation w.r.t. $\Zn \sim \task^n$ and $z_i' \sim \task$ of the left side member above, according to
\cite[Lemma $7$]{bousquet2002stability}, we get 
\begin{equation*}
\Exp_{\Zn \sim \task^n} ~ \Exp_{z_i' \sim \task} ~ \Big[ \ell (  \langle x_i', \hat w_\theta(\Zn)  \rangle, y_i' ) 
- \ell (  \langle x_i', \hat w_\theta(\Zn^{(i)})  \rangle, y_i' ) \Big] =
\Exp_{\Zn \sim \task^n} ~ \Big[ \cR_\task ( \hat w_\theta(\Zn) ) - \cR_{\Zn} ( \hat w_\theta(\Zn) ) \Big].
\end{equation*}
Finally, taking the expectation of the right side member, exploiting the fact that the points are i.i.d. according $\task$,
we get
\begin{equation}
\Exp_{\Zn \sim \task^n} ~ \Exp_{z_i' \sim \task} ~
\frac{1}{\la n} \Bigg ( \big \| x_i' s_{\theta,i}' \big \|^2 
+ \big \| x_i' s_{\theta,i}' \big \| \big \| x_i s_{\theta,i} \big \| 
\Bigg ) \le \frac{2}{\la n} ~ \Exp_{\Zn \sim \task^n} 
~ \Exp_{z_i' \sim \task} ~ \big \| x_i' s_{\theta,i}' \big \|^2,
\end{equation}
where we recall that $s_{\theta,i}' \in \partial \ell ( \cdot , y_i' ) 
(  \langle x_i', \hat w_\theta(\Zn)  \rangle )$.
The statement derives from combining the two last statements above with 
the expectation w.r.t. $\Zn \sim \task^n$ and $z_i' \sim \task$ of \cref{mela}.
The second statement directly derives from the first one, once one 
observes that, if $\ell(\cdot,y)$ is $L$-Lipschitz for any $y \in \Y$, then,
$| s_{\theta,i}' | \le L$ (see \cite[Lemma $14.7$]{shalev2014understanding}).
\end{proof}


\section{Examples}
\label{examples}

In this section, we provide the deduction of the statements in 
the examples reported in \cref{unconditional_for_bias_advantage}.
We start from presenting some computation regarding a generic 
environment parametrized by a latent variable in \cref{general_ex_deduction} 
and, then, in \cref{clusters_ex_deduction}, we specify this computation 
and we derive the statement in \cref{clusters_ex}. Finally, in 
\cref{curve_ex_deduction}, we prove the statement in \cref{curve_ex}.

\subsection{General parametrization}
\label{general_ex_deduction}

Consider the case where a latent variable $\alpha \in \A$ 
parametrizes the environment $\env$. Denote by $\env(\cdot|\alpha)$ 
the conditional distributions given $\alpha$ and by $\env_\A$ 
the marginal distribution of the latent variable. As usual, we assume
$\rho(\task, \alpha) = \rho(\task | \alpha) \rho_\A(\alpha)$.
Introduce also
\begin{equation}
w(\alpha) = \int w_\task~\rho(\task|\alpha)~dw_\task
\quad \quad \quad 
\sigma(\alpha)^2 = \frac{1}{2} \int \nor{w_\task - w(\alpha)}^2 
~\rho(\task|\alpha)~dw_\task
\end{equation}
the conditional expectation and the conditional variance of the target 
weight vectors $w_\task$ given $\alpha$, respectively. 
We now explicitly compute the unconditional and the 
conditional variance for this generic environment. \\

\paragraph{Unconditional variance}
We start from observing that thanks to the parametrization of the 
environment $\env$, we can rewrite the unconditional variance
as follows
\begin{equation} \label{unc_var_computation_gen}
\begin{split}
{\rm Var}_\env(w_\env)^2 & = 
\Exp_{\task \sim \mt} ~ \big \| w_\task - w_\env \big \|^2
= \int \nor{w_\task - w_\env}^2 ~\rho(\task)~dw_\task\\
& = \int \Bigg( \int \nor{w_\mu - w_\env}^2 ~\rho(\task | \alpha) 
~dw_\task \Bigg) \rho_\A(\alpha)~d\alpha.
\end{split}
\end{equation} 
We now observe that, for any $\alpha \in \A$, we can write 
the following
\begin{equation}
\begin{split}
\int \nor{w_\mu - w_\env}^2& ~\rho(\task | \alpha) ~dw_\task \\ 
& = \int \Big( \nor{w_\task}^2 - 2\scal{w_\task}{w_\env} + \nor{w_\env}^2 \Big)
~\rho(\task | \alpha)~dw_\task\\
& = \int \nor{w_\task}^2 ~\rho(\task | \alpha)~dw_\task - 2 \scal{w(\alpha)}{w_\env} 
+ \nor{w_\env}^2 \\
& = \int \nor{w_\task}^2~\rho(\task | \alpha)~dw_\task \pm \nor{w(\alpha)}^2
- 2\scal{w(\alpha)}{w_\env} + \nor{w_\env}^2\\
& = \int \big \| w_\task - w(\alpha) \big \|^2~\rho(\task | \alpha)~dw_\task 
+ \nor{w(\alpha) - w_\env}^2 \\
& = 2 \sigma(\alpha)^2 + \nor{w(\alpha) - w_\env}^2.
\end{split}
\end{equation} 
Hence, substituting in \cref{unc_var_computation_gen}, we get
\begin{equation} \label{uncond_var_closed_form_gen_bis}
{\rm Var}_\env(w_\env)^2 = 2 \int  \sigma(\alpha)^2~\rho_\A(\alpha)~d\alpha 
+ \int \nor{w(\alpha) - w_\env}^2~\rho_\A(\alpha)~d\alpha.
\end{equation} 
We now observe that the second term above can be rewritten as follows
\begin{equation}
\begin{split}
\int \nor{w(\alpha) - w_\env}&^2~\rho_\A(\alpha)~d\alpha \\
& = \int \nor{w(\alpha)}^2\rho_\A(\alpha)~d\alpha - \nor{w_\env}^2\\
& = \int \nor{w(\alpha)}^2\rho_\A(\alpha)~d\alpha - \nor{\int w(\alpha')\rho_\A(\alpha')~d\alpha'}^2\\
& = \int \nor{w(\alpha)}^2\rho_\A(\alpha)~d\alpha - \int \scal{w(\alpha)}{w(\alpha')}\rho_\A(\alpha)\rho_\A(\alpha')~d\alpha~d\alpha'\\
& = \int \Big(\nor{w(\alpha)}^2 -  \scal{w(\alpha)}{w(\alpha')}\Big)\rho_\A(\alpha)\rho_\A(\alpha')~d\alpha~d\alpha'.
\end{split}
\end{equation}
But, since
\begin{equation}
\int \nor{w(\alpha)}^2\rho_\A(\alpha)\rho_\A(\alpha')~d\alpha 
= \frac{1}{2}\int \big( \nor{w(\alpha)}^2 + \nor{w(\alpha')}^2 \big)
\rho_\A(\alpha')\rho_\A(\alpha') ~d\alpha~d\alpha',
\end{equation}
we conclude 
\begin{equation}
\int \nor{w(\alpha) - w_\env}^2~\rho_\A(\alpha)~d\alpha 
= \frac{1}{2}\int \nor{w(\alpha)-w(\alpha')}^2~\rho_\A(\alpha)
\rho_\A(\alpha')~d\alpha~d\alpha'.
\end{equation}
Hence, substituting in \cref{uncond_var_closed_form_gen_bis}, we get
\begin{equation} \label{uncond_var_closed_form_final}
{\rm Var}_\env(w_\env)^2 
= 2 \int  \sigma(\alpha)^2~\rho_\A(\alpha)~d\alpha + 
\frac{1}{2}\int \nor{w(\alpha)-w(\alpha')}^2~\rho_\A(\alpha)
\rho_\A(\alpha')~d\alpha~d\alpha'.
\end{equation} 

\paragraph{Conditional variance}
We now focus on the conditional variance. As explained in 
\cref{clusters_ex}, also in this case, we consider as side 
information a set of new features $X = (x_i)_{i = 1}^n 
\in \cup_{n \in \N} \X^n$. As a consequence, we focus
on conditioning functions of the form $\tau: \cup_{n \in \N} 
\X^n \to \Real^d$. From \cref{oracle_function}, we know
that the ideal function $\tau_\env: \cup_{n \in \N} \X^n \to \Real^d$ 
minimizing the conditional variance term over the space $\T$ 
of the measurable functions, is characterized, for almost every  
$X \in \cup_{n \in \N} \X^n$, by
\begin{equation} \label{regression_func_closed_form_0}
\tau_\env(X) = \Exp_{\task\sim\rho(\cdot|X)} ~ w_\task 
= \int w_\task ~ \rho(\task| X)~dw_\task.
\end{equation}
We now observe that thanks to the parametrization of the 
environment $\env$, for any target weight vector $w_\task$
and features' set $X$, we can write
\begin{equation} \label{banana}
\begin{split}
\rho(\task|X) & = \frac{\rho(\task,X)}{\rho_\X(X)} 
= \frac{\int \rho(\task,X,\alpha)~d\alpha}{\rho_\X(X)} 
= \int \rho(\task|X,\alpha)\frac{\rho(X,\alpha)}{\rho_\X(X)}~d\alpha \\
& = \int \rho(\task|X,\alpha)\rho(\alpha|X)~d\alpha 
= \int \rho(\task|\alpha)\rho(\alpha|X)~d\alpha,
\end{split}
\end{equation}
where, in the last equality, we have exploited the fact that,
by construction, $\task$ is conditionally independent to 
$X$ w.r.t. $\alpha$, namely $\rho(\task|X,\alpha) = 
\rho(\task|\alpha)$. Then, substituting in 
\cref{regression_func_closed_form_0}, we get
\begin{equation} \label{regression_func_closed_form}
\begin{split}
\tau_\env(X) & = \int w_\task~\rho(\task|X)~dw_\task\\
& = \int w_\task \left(\int \rho(\task|\alpha)\rho(\alpha|X)~d\alpha\right)~dw_\task\\
& = \int \left(\int w_\task~\rho(\task|\alpha)~dw_\task\right)~\rho(\alpha|X)~d\alpha\\
& = \int w(\alpha)~\rho(\alpha|X)~d\alpha.
\end{split}
\end{equation}

\begin{remark}[\cref{ass_2} in this example]
From the expression above, we can conclude that the function $\tau_\env$ in 
\cref{regression_func_closed_form} is a smooth function of $X$, if $\rho(X|\alpha)$ 
is a smooth function of $X$ for any $\alpha \in \A$. This means that, in such a case, 
there exist a Reproducing Kernel Hilbert Space (RKHS) $\mathcal{H}$ such 
that $\tau_\env \in \mathcal{H}$ and, consequently, making \cref{ass_2} satisfied. 
For instance, we can take $\mathcal{H}$ to be the space induced by the Abel kernel 
\begin{equation}
k(X,X') = e^{-\sum_{j=1}^n \frac{\| x_j - x_j' \|}{\sigma}},
\quad \quad \sigma > 0 
\quad \quad X,X' \in \cup_{n \in \N} \X^n.
\end{equation}
In this case, $\mathcal{H} = W^{d/2+1,2}$ corresponds the Sobolev's 
space of functions with square integrable $d/2+1$ derivatives. 
\end{remark}

We now proceed by computing the conditional variance.
In order to do this, we observe that
\begin{equation} \label{final_goal_conditional_var}
{\rm Var}_\env(\tau_\env)^2 = \Exp_{(\task, X) \sim \env} ~ 
\big \| w_\task - \tau_\env(X) \big \|^2
= \Exp_{X \sim \env_\X} ~ \Exp_{\task \sim \env(\cdot|X)} ~
\big \| w_\task - \tau_\env(X) \big \|^2.
\end{equation}
We now observe that, for any set of features $X$, 
exploiting \cref{banana}, we can rewrite 
the inner expectation above as follows
\begin{equation} \label{conditional_var_1}
\begin{split}
\Exp_{\task \sim \env(\cdot|X)} ~ \big \| w_\task - \tau_\env(X) \big \|^2
& = \int \nor{w_\task - \tau_\env( X)}^2 \rho(\task | X)~dw_\task \\
& = \int \nor{w_\task - \tau_\env(X)}^2~\left(\int \rho(\task|\alpha)\rho(\alpha|X)
~d\alpha\right)~dw_\task\\
& = \int \left(\int \nor{w_\task -\tau_\env(X)}^2~\rho(\task|\alpha)~dw_\task \right)
~\rho(\alpha|X)~d\alpha.
\end{split}
\end{equation}
But, for each $\alpha \in \A$, we can write
\begin{equation}
\begin{split}
\int & \nor{w_\task - \tau_\env(X)}^2~\rho(\task|\alpha)~dw_\task \\ 
& = \int \nor{w_\task}^2~\rho(\task|\alpha)~dw_\task 
-2 \scal{w(\alpha)}{\tau_\env(X)} + \nor{\tau_\env(X)}^2\\
& = \int \nor{w_\task}^2~\rho(\task|\alpha)~dw_\task \pm \nor{w(\alpha)}^2 
- 2\scal{w(\alpha)}{\tau_\env(X)} + \nor{\tau_\env(X)}^2\\
& = 2 \sigma(\alpha)^2 + \nor{w(\alpha) - \tau_\env(X)}^2.
\end{split}
\end{equation}
Hence, substituting into \cref{conditional_var_1}, we get
\begin{equation} 
\Exp_{\task \sim \env(\cdot|X)} ~ \big \| w_\task - \tau_\env(X) \big \|^2 
= 2 \int \sigma(\alpha)^2~\rho(\alpha|X)~d\alpha + 
\int \nor{w(\alpha) - \tau_\env(X)}^2~\rho(\alpha|X)~d\alpha.
\end{equation}
Hence, integrating w.r.t. $X$, we get
\begin{equation}
\begin{split}
{\rm Var}_\env(\tau_\env)^2
& = \Exp_{X \sim \env_\X} ~ \Exp_{\task \sim \env(\cdot|X)} ~
\big \| w_\task - \tau_\env(X) \big \|^2 \\
& = 2 \int \sigma(\alpha)^2~\rho(\alpha|X) \rho_\X(X)~d\alpha~dX 
+ \int \nor{w(\alpha)-\tau_\env(X)}^2~\rho(\alpha|X) \rho_\X(X)~d\alpha~dX\\
& = 2 \int \sigma(\alpha)^2~\rho_\A(\alpha)~d\alpha + \int \nor{w(\alpha) - \tau_\env(X)}^2 \rho(\alpha|X) \rho_\X(X)~d\alpha~dX.
\end{split}
\end{equation}
We now observe that, exploiting the closed form of $\tau_\env$ in 
\cref{regression_func_closed_form}, the second term above can be 
rewritten as follows
\begin{equation} \label{minni}
\begin{split}
\int & \nor{w(\alpha) - \tau_\env(X)}^2 \rho(\alpha|X) \rho_\X(X)~d\alpha~dX \\
& = \int \nor{w(\alpha) - \int w(\alpha')~\rho(\alpha'|X)~d\alpha'}^2
~\rho(\alpha|X)\rho_\X(X)~d\alpha~dX \\
& = \int \nor{w(\alpha)}^2~\rho(\alpha|X)\rho_\X(X)~d\alpha~dX 
- 2\int \scal{w(\alpha)}{w(\alpha')}~\rho(\alpha|X)\rho(\alpha'|X)
\rho_\X(X)~d\alpha~d\alpha'~dX \\
& \quad + \int \nor{\int w(\alpha')~\rho(\alpha'|X)~d\alpha'}^2\rho_\X(X)~dX.
\end{split}
\end{equation}
Note now that
\begin{equation} \label{pippo}
\begin{split}
\int & \nor{w(\alpha)}^2 ~\rho(\alpha|X)\rho_\X(X)~d\alpha~dX \\
& = \frac{1}{2}\Bigg(\int \nor{w(\alpha)}^2 ~\rho(\alpha|X)\rho_\X(X) 
~d\alpha~dX  
+ \int \nor{w(\alpha')}^2 ~\rho(\alpha'|X)\rho_\X(X)~d\alpha'~dX \Bigg)\\
& = \frac{1}{2}\Bigg(\int \Big( \nor{w(\alpha)}^2 +  \nor{w(\alpha')}^2 \Big) ~\rho(\alpha|X)\rho(\alpha'|X)\rho_\X(X) ~d\alpha~d\alpha'~dX \Bigg)
\end{split}
\end{equation}
and 
\begin{equation} \label{pluto}
\int \nor{\int w(\alpha')~\rho(\alpha'|X)~d\alpha'}^2\rho_\X(X)~dX 
= \int \scal{w(\alpha)}{w(\alpha')}~\rho(\alpha|X)\rho(\alpha'|X)\rho_\X(X)
~d\alpha~d\alpha'~dX.
\end{equation}
Substituting \cref{pippo} and \cref{pluto} in \cref{minni}, we get
\begin{equation}
\begin{split}
\int & \nor{w(\alpha) - \tau_\env(X)}^2~\rho(\alpha|X) \rho_\X(X)~d\alpha~dX \\
& = \int \frac{1}{2}\Big(\nor{w(\alpha)}^2 - 2 \scal{w(\alpha)}{w(\alpha')} + \nor{w(\alpha')}^2 \Big)~\rho(\alpha|X)\rho(\alpha'|X)\rho_\X(X)~d\alpha~d\alpha'~dX\\
& = \frac{1}{2}\int \nor{w(\alpha)-w(\alpha')}^2 \rho(\alpha|X) \rho(\alpha'|X)\rho_\X(X) 
~ d\alpha~d\alpha'~dX\\
& = \frac{1}{2}\int \nor{w(\alpha)-w(\alpha')}^2\frac{\rho(X|\alpha) \rho(X|\alpha')}{\rho_\X(X)}~\rho_\A(\alpha) \rho_\A(\alpha')~d\alpha~d\alpha'~dX\\
& = \frac{1}{2} \int \nor{w(\alpha)-w(\alpha')}^2\Bigg(\int\frac{\rho(X|\alpha)
\rho(X|\alpha')}{\rho_\X(X)}~dX\Bigg)\rho_\A(\alpha)\rho_\A(\alpha')~d\alpha~d\alpha'.
\end{split}
\end{equation}
Hence, the conditional variance is given by
\begin{equation} \label{cond_var_closed_form_final}
\begin{split}
{\rm Var}_\env&(\tau_\env)^2
= 2 \int \sigma(\alpha)^2~\rho(\alpha)~d\alpha \\
& \quad + \frac{1}{2} \int \nor{w(\alpha)-w(\alpha')}^2\Bigg(\int\frac{\rho(X|\alpha)
\rho(X|\alpha')}{\rho_\X(X)}~dX\Bigg)\rho_\A(\alpha)\rho_\A(\alpha')~d\alpha~d\alpha'.
\end{split}
\end{equation}

\paragraph{Conditional vs unconditional variance}
Subtracting \cref{cond_var_closed_form_final} to 
\cref{uncond_var_closed_form_final}, we get that
the difference between the unconditional and 
conditional variance is given by the following 
closed form
\begin{equation} \label{gap_closed_form_final}
\begin{split}
{\rm Var}_\env&(w_\env)^2 - {\rm Var}_\env(\tau_\env)^2 \\
& = \frac{1}{2} \int \Bigg(1 - \int\frac{\rho(X|\alpha)
\rho(X|\alpha')}{\rho_\X(X)}~dX\Bigg) \nor{w(\alpha)-w(\alpha')}^2
\rho_\A(\alpha)\rho_\A(\alpha')~d\alpha~d\alpha'.
\end{split}
\end{equation}
Hence, if 
\begin{equation}
\int \frac{\rho(X|\alpha) \rho(X|\alpha')}{\rho_\X(X)}~dX \le \epsilon(\alpha, \alpha')
\end{equation}
for some $\epsilon: \A\times \A \to \Real_+$, we can write
\begin{equation}
{\rm Var}_\env(w_\env)^2 - {\rm Var}_\env(\tau_\env)^2 
\ge \frac{1}{2} \int \Big(1 - \epsilon(\alpha, \alpha') \Big) 
\nor{w(\alpha)-w(\alpha')}^2\rho_\A(\alpha)
\rho_\A(\alpha')~d\alpha~d\alpha'.
\end{equation}


\subsection{Clusters (\cref{clusters_ex})}
\label{clusters_ex_deduction}

The example in the section above encompasses the setting outlined
in \cref{clusters_ex}, by identifying the latent variable $\alpha$ with 
the clusters' indexes, namely, $\A = \{1,\dots,m\}$ and, for any 
$\alpha \in \A$, $\env_\A(\alpha) = 1/m$. We now show that adapting
the results above to this specific setting, we manage to show the
statement in \cref{clusters_ex} in the main body.\\

\paragraph{Unconditional variance} Specifying \cref{uncond_var_closed_form_final}
to the setting outlined in \cref{clusters_ex}, we get the following closed
form for the unconditional variance:
\begin{equation} 
{\rm Var}_\env(w_\env)^2 
= \frac{2}{m} \sum_{\alpha = 1}^m \sigma(\alpha)^2 + 
\frac{1}{2 m ^2} \sum_{\alpha, \alpha' = 1}^m \nor{w(\alpha)-w(\alpha')}^2.
\end{equation} 

\paragraph{Conditional variance} Specifying \cref{cond_var_closed_form_final}
to the setting outlined in \cref{clusters_ex}, we get the following closed
form for the conditional variance:
\begin{equation}
\begin{split}
{\rm Var}_\env(\tau_\env)^2
= \frac{2}{m} \sum_{\alpha = 1}^m \sigma(\alpha)^2
+ \frac{1}{2 m ^2} \sum_{\alpha, \alpha' = 1}^m 
\Bigg(\int\frac{\rho(X|\alpha) \rho(X|\alpha')}{\rho_\X(X)}~dX\Bigg)
\nor{w(\alpha)-w(\alpha')}^2.
\end{split}
\end{equation}

\paragraph{Conditional vs unconditional variance}
Finally, specifying \cref{gap_closed_form_final}
to the setting outlined in \cref{clusters_ex}, we get the 
following closed form for the gap between the unconditional 
and the conditional variance:
\begin{equation} \label{opo}
\begin{split}
{\rm Var}_\env(w_\env)^2 - {\rm Var}_\env(\tau_\env)^2 
& = \frac{1}{2 m ^2} \sum_{\alpha, \alpha' = 1}^m 
\Bigg(1 - \int\frac{\rho(X|\alpha) \rho(X|\alpha')}{\rho_\X(X)}~dX\Bigg)
\nor{w(\alpha)-w(\alpha')}^2.
\end{split}
\end{equation}

The last ingredient we need to prove the upper bound in 
\cref{clusters_ex} is the following.

\begin{proposition}
Assume now that for any $\alpha \in \A = \{1,\dots,m\}$, $\rho(X | \alpha)$ 
is a Gaussian distribution with mean $x(\alpha) \in \Real^d$ and variance 
$\sigma_\X^2$. Then, for any $\alpha, \alpha' \in \A$,
\begin{equation} \label{oppo}
\int\frac{\rho(X|\alpha) \rho(X|\alpha')}{\rho_\X(X)}~dX
\le \frac{m}{2} ~ e ^{- \frac{n}{\sigma_\X^2}
\nor{x(\alpha) - x(\alpha')}^2}.
\end{equation}
\end{proposition}

\begin{proof}
Thanks to the composition of the environment in clusters, we can write
\begin{equation}
\rho(X) = \sum_{\epsilon = 1}^m \rho(X|\epsilon)\rho_\A(\epsilon) 
= \frac{1}{m}\sum_{\epsilon = 1}^m \rho(X|\epsilon).
\end{equation}
As a consequence, for any $\alpha, \alpha' \in \A$, we can write
\begin{equation} \label{initial_step}
\begin{split}
\int\frac{\rho(X|\alpha)\rho(X|\alpha')}{\rho_\X(X)}~dX
& = m \int\frac{\rho(X|\alpha)\rho(X|\alpha')}{\sum_{\epsilon=1}^m\rho(X|\epsilon)}~dX \\
& \leq m \int\frac{\rho(X|\alpha)\rho(X|\alpha')}{\rho(X|\alpha) + \rho(X|\alpha')}~dX \\
& \leq \frac{m}{2} \int \sqrt{\rho(X|\alpha)\rho(X|\alpha')}~dX,
\end{split}
\end{equation}
where in the last inequality we have used the inequality
\begin{equation}
\frac{ab}{a+b} \leq \frac{\sqrt{ab}}{2},
\end{equation}
holding for any $a,b >0$.
We now observe that, by assumption, we are considering Gaussian 
distributions for the inputs' probability, i.e., for any $\alpha \in \{1, \dots, k \}$,
we have
\begin{equation}
\rho(X|\alpha) = \prod_{j=1}^n \frac{1}{\sqrt{2\pi\sigma_{\X}^2}}
~ e^{- \frac{\| x_j - x(\alpha) \|^2}{\sigma_{\X}^2}}.
\end{equation}
Hence, we have
\begin{equation} \label{first_step}
\begin{split}
\int \sqrt{\rho(X|\alpha)\rho(X|\alpha')}~dX 
= \frac{1}{\Pi_{j=1}^n \sqrt{2 \pi \sigma_{\X}^2}}
\int e^{- \frac{1}{\sigma_{\X}^2} \sum_{j=1}^n
\nor{x_j - x(\alpha)}^2 + \nor{x_j - x(\alpha')}^2}~\Pi_j d x_j.
\end{split}
\end{equation}
We now observe that
\begin{equation} \label{second_step}
\begin{split}
& \nor{x_j-x(\alpha)}^2 + \nor{x_j - x(\alpha')}^2 \\
& \quad = 2\nor{x_j}^2 - 2 \scal{x_j}{x(\alpha) + x(\alpha')} 
+ \nor{x(\alpha)}^2 + \nor{x(\alpha')}^2\\
& \quad = 2\nor{x_j}^2 - 2 \scal{x_j}{x(\alpha) + x(\alpha')} 
+ \nor{x(\alpha)}^2 + \nor{x(\alpha')}^2 \pm \frac{1}{2}\nor{x(\alpha) + x(\alpha')}^2 \\
& \quad = \nor{\sqrt{2}x_j - \frac{1}{\sqrt{2}} (x(\alpha) + x(\alpha'))}^2 -\frac{1}{2}\nor{x(\alpha) + x(\alpha')}^2 + \nor{x(\alpha)}^2 + \nor{x(\alpha')}^2\\
& \quad = \nor{\sqrt{2}x_j - \frac{1}{\sqrt{2}} (x(\alpha) +x(\alpha'))}^2 + \frac{1}{2}\nor{x(\alpha)}^2 + \frac{1}{2}\nor{x(\alpha')}^2 - \scal{x(\alpha)}{x(\alpha')}\\
& \quad = \nor{\sqrt{2}x_j - \frac{1}{\sqrt{2}} (x(\alpha) + x(\alpha'))}^2 + \frac{1}{2}\nor{x(\alpha) - x(\alpha')}^2.
\end{split}
\end{equation}
Substituting \cref{second_step} into \cref{first_step}, we conclude
\begin{equation}
\begin{split}
& \int \sqrt{\rho(X|\alpha)\rho(X|\alpha')}~dX \leq \\
& \leq e^{- \frac{n}{\sigma_\X^2}\nor{x(\alpha) - x(\alpha')}^2}~\frac{1}{\Pi_{j=1}^n\sqrt{2\pi\sigma_\X^2}}~\int  e^{- \frac{1}{\sigma_\X^2}\sum_{j=1}^n\nor{\sqrt{2}x_j - \frac{1}{\sqrt{2}} (x(\alpha) + x(\alpha'))}^2}~\Pi_i d x_i\\
& = e^{- \frac{n}{\sigma_\X^2}\nor{x(\alpha) - x(\alpha')}^2}~\frac{1}{\Pi_{j=1}^n\sqrt{2\pi\sigma_\X^2}}~\int e^{- \frac{1}{\sigma_\X^2}\sum_{j=1}^n\nor{x_j - \frac{x(\alpha) + x(\alpha')}{2}}^2}~\Pi_i d x_i\\
& = e^{- \frac{n}{\sigma_\X^2}\nor{x(\alpha) - x(\alpha')}^2},
\end{split}
\end{equation}
where in the last equality we have exploited the integral of the Gaussian 
distribution $\mathcal{N}\Big(\frac{x(\alpha) + x(\alpha')}{2},\sigma_\X\Big)$:
\begin{equation}
\frac{1}{\Pi_{j=1}^n\sqrt{\pi\sigma_\X^2}}~\int e^{- \frac{1}{2\sigma_\X^2}
\sum_{j=1}^n\nor{x_j - \frac{x(\alpha) + x(\alpha')}{2}}^2}~\Pi_i d x_i = 1.
\end{equation}
Using the last inequality above in \cref{initial_step}, we get the desired
statement.
\end{proof}
The desired statement in \cref{clusters_ex} derives from 
combining \cref{opo} with \cref{oppo}.


\subsection{Circle (\cref{curve_ex})}
\label{curve_ex_deduction}

Consider now the setting outlined in \cref{curve_ex}. We proceed as before:
we first compute the unconditional variance, then, the conditional variance 
and, finally, the gap between them.\\

\paragraph{Unconditional variance} We start from observing that, 
since by construction, for any $s \in [0,1]$, $\env(\task|s)$ is the 
Gaussian distribution with mean $h(s)$, $\ms$ is the uniform 
distribution on $[0,1]$ and $h$ is centered in $c$, then, we have
\begin{equation}
w_\rho = \Exp_{\task \sim \rho} ~ w_\task 
= \Exp_{s \sim \ms} ~ \Exp_{\task \sim \env(\cdot|s)} ~ w_\task 
= \Exp_{s \sim \ms} ~ h(s)
= c.
\end{equation}
Hence, we can rewrite the unconditional variance as follows
\begin{equation} \label{uncond_var_closed_form_final_circle}
\begin{split}
{\rm Var}_\env(w_\env)^2
& = \Exp_{(\mu,s)\sim \rho} ~ \|w_\task - c \|^2 \\
& = \int \|w_\task - c \pm h(s)\|^2 ~ \rho(\task,s) ~ dw_\task ds \\
& = \int \Bigg( \int \|w_\task - h(s) \|^2 ~ \rho(\task|s) ~ dw_\task \Bigg) ~ \rho_\Ss(s) ~ ds
+ \int \| h(s) - c \|^2 ~ \rho_\Ss(s) ~ ds \\
& \quad + \int \scal{c - h(s)}{\int (w_\task - h(s)) ~ \env(\task|s) ~ dw_\task } 
\ms(s) ~ ds \\
& = \sigma^2 + r^2,
\end{split}
\end{equation}
where, in the last equality, we have exploited the fact  
$\| h(s) - c \| = r$ for any $s \in \Ss$ and the fact that,
thanks to the assumption 
$\env(\cdot|s) = \mathcal{N}(h(s), \sigma^2 I)$,
\begin{equation}
\int (w_\task - h(s)) ~ \env(\task|s) ~ dw_\task = 0
\quad \quad \quad 
\int \|w_\task - h(s) \|^2 ~ \rho(\task|s) ~ dw_\task = \sigma^2.
\end{equation}

\paragraph{Conditional variance} 
Since by construction $\env(\cdot|s) = \mathcal{N}(h(s), \sigma^2 I)$,
we immediately see that the ideal function $\tau_\env: [0,1] \to \Real^d$ 
in \cref{oracle_function} and the corresponding conditional variance 
can be, respectively, rewritten as follows
\begin{equation}
\tau_\env(s) = \Exp_{\task\sim\rho(\cdot|s)} ~ w_\task = h(s)
\end{equation}
\begin{equation} \label{cond_var_closed_form_final_circle}
{\rm Var}_\env(\tau_\env)^2 = 
\Exp_{(w_\task,s) \sim \rho} ~ \|w_\task - h(s)\|^2 
= \int \Bigg( \int \|w_\task - h(s)\|^2 ~ \env(\task|s) ~ dw_\task \Bigg) 
~ \ms ~ ds = \sigma^2.
\end{equation}

\paragraph{Conditional vs unconditional variance}
Subtracting \cref{cond_var_closed_form_final_circle} to 
\cref{uncond_var_closed_form_final_circle}, we get that
the difference between the unconditional and 
conditional variance is given by
\begin{equation} \label{opo_circle}
{\rm Var}_\env(w_\env)^2 - {\rm Var}_\env(\tau_\env)^2 
= r^2.
\end{equation}

All the statements given in \cref{curve_ex} have hence been proven.


\section{Closed forms for \cref{ass_2} }
\label{proof_oracle_meta_parameters}

Thanks to \cref{ass_2}, we know that there exist 
$M_\env \in \Real^{d \times k}$ and $b_\env \in \Real^d$
such that $\tau_\env(\cdot) = M_\env \Phi(\cdot) + b_\env$. 
In the following lemma, we give the closed form 
of these quantities and the corresponding variance. 
We let $\tr(\cdot)$ and $\cdot^*$ be the trace and the 
conjugate operators respectively.

\begin{restatable}[Best linear conditioning function in hindsight]{lemma}{OracleMetaParameters} \label{oracle_meta_parameters}
Recall the vector $w_\env = \Exp_{\task \sim \mt} ~ w_\task$ and introduce
the vector $\nu_\env = \Exp_{s \sim \ms} ~ \Phi(s)$. Introduce also 
the following covariance matrices
\begin{equation}
{\rm Cov}_{\env}(s,s) = \Exp_{s \sim \ms} ~ \Big[ \bigl( \Phi(s) - \nu_\env \bigr) 
\bigl( \Phi(s) - \nu_\env \bigr) \trans \Big] \in \Real^{k \times k}
\end{equation}
\begin{equation}
{\rm Cov}_{\env}(w,w) = \Exp_{\task \sim \mt} ~ 
\Big[ \bigl( w_\task - w_\env \bigr) 
\bigl( w_\task - w_\env \bigr) \trans \Big] \in \Real^{d \times d}
\end{equation}
\begin{equation}
{\rm Cov}_{\env}(w,s) = \Exp_{(\task, s) \sim \env} ~ 
\Big[  \bigl( w_\task - w_\env \bigr)
\bigl( \Phi(s) - \nu_\env \bigr)\trans \Big] \in \Real^{d \times k}.
\end{equation}
Then, 
\begin{equation} \label{banana2}
\begin{split}
\min_{M \in \Real^{d \times k}, b \in \Real^d}
{\rm Var}_\env(\tau_{M,b})^2  
& = {\rm Var}_\env(w_\env)^2 
- \tr \Big( {\rm Cov}_{\env}(w,w) {\rm Corr}_\env(w,s) 
\trans {\rm Corr}_\env(w,s) \Big) \\
& = {\rm Var}_\env(w_\env)^2  - \big \| 
{\rm Cov}_{\env}(s,s)^{1/2} M_\env \big \|_F^2,
\end{split}
\end{equation}
where we have introduced the correlation matrix 
\begin{equation}
{\rm Corr}_\env(w,s) = 
{\rm Cov}_{\env}(w,w)^{\dagger/2} ~ 
{\rm Cov}_{\env}(w,s)
{\rm Cov}_{\env}(s,s)^{\dagger/2} 
\in \Real^{d \times k}.
\end{equation}
Moreover, the (minimum norm) values at which the minimum above is attained 
are given by
\begin{equation} \label{oracle_M}
M_\env = {\rm Cov}_{\env}(w,s) {\rm Cov}_{\env}(s,s)^\dagger
\end{equation}
\begin{equation} \label{oracle_b}
b _\env = w_\env - {\rm Cov}_{\env}(w,s) 
{\rm Cov}_{\env}(s,s)^\dagger ~ \nu_\env.
\end{equation}
\end{restatable}

When \cref{ass_2} holds, the minimum conditional 
variance in \cref{oracle_function} can be rewritten as 
$\min_{\tau \in \T} ~ {\rm Var}_\env(\tau)^2
= \min_{\tau \in \T_\Phi} ~ {\rm Var}_\env(\tau)^2$.
As a consequence, in this case, the statement above in \cref{banana2} 
allows us to express the gap between the conditional and 
the uncoditional variance in \cref{gap_conditional_unconditional} 
as a function of the correlation between the target tasks' weight 
vectors and the side information. In addition, we can 
also deduce that such a gap is significant when the `inclination' 
of the linear relation linking the target tasks' weight vectors 
and the side information (more formally, $\| 
{\rm Cov}_{\env}(s,s)^{1/2} M_\env \|_F^2$) is large. 
This is not surprising, since, in this case, the gap between conditional 
and unconditional meta-learning can be interpreted as the gap 
in using the best linear function w.r.t. the constant one 
$\tau \equiv w_\env$.

As we will see in the following, the proof of \cref{oracle_meta_parameters}, 
directly derives from the following facts regarding linear Least Squares.

\begin{lemma} \label{extracting_cov}
Let $\X$ be an Hilbert space, $\Y = \Real^d$ and $\H = \Real^k$. 
Consider a map $\Psi: \X \to \H$ and a joint probability distribution 
$\rho$ on $\X \times \Y$ with conditional distribution $\rho(x|y)$ 
and marginal $\rho_\Y(y)$. Denote by $\otimes$ the standard outer 
product, introduce the covariance operators:
\begin{equation}
C_{yy} = \Exp ~ [y \otimes y]
\quad \quad \quad 
C_{xx} = \Exp ~ [\Psi(x) \otimes \Psi(x)]
\quad \quad \quad 
C_{xy} = \Exp ~ [y \otimes \Psi(x)]
\end{equation}
and the correlation operator
\begin{equation}
{\rm Corr}_{xy} = C_{xx}^{\dagger/2}  C_{xy} C_{yy}^{\dagger/2}.
\end{equation}
Then, 
\begin{equation}
\begin{split}
\min_{M \in \Real^{d \times k}} ~
\Exp_{(x,y) \sim \env} ~ \big \| y - M \Psi(x) \big \|^2
& = \tr \big( C_{yy} \big) - \tr \big( C_{yy} {\rm Corr}_{xy}^* 
{\rm Corr}_{xy} \big) \\
& = \tr \big( C_{yy} \big) - \big \| C_{xx}^{1/2} M_\env \big \|_F^2.
\end{split}
\end{equation}
The optimal (minimum norm) matrix $M_\env$ 
is given by
\begin{equation}
M_\env = C_{xy} C_{xx}^\dagger.
\end{equation}
\end{lemma}

\begin{proof}
For any $M \in \Real^{d \times k}$, we can rewrite
\begin{equation} \label{kj}
\begin{split}
\Exp_{(x, y) \sim \env} ~ \big \| y - M \Psi(x) \big \|^2 
& = \Exp_{y \sim \env_\Y} ~\| y \|^2 
+ \Exp_{x \sim \env_\X} ~ \big \| M \Psi(x) \big \|^2
- 2 ~ \Exp_{(x, y) \sim \env} ~ \big \langle y, M \Psi(x) \big \rangle \\
& = \tr \big( C_{yy} \big) + \tr \big( C_{xx} M^* M \big) 
- 2 \tr \big( C_{xy}^* M \big).
\end{split}
\end{equation}
By setting the derivatives w.r.t. $M$ equal to zero, 
we know that the optimal matrix $M_\env$ satisfies
\begin{equation} \label{optimality_M_gen}
M_\env C_{xx} = C_{xy}.
\end{equation}
Hence, the optimal (minimum norm) matrix $M_\env$ 
is given by
\begin{equation} \label{optimal_M_gen}
M_\env = C_{xy} C_{xx}^\dagger.
\end{equation}
We now compute the corresponding minimum value. 
We first observe that, by the closed form of the optimal 
matrix $M_\env$, we can rewrite
\begin{equation} \label{giulia}
\begin{split}
\tr \big( C_{xx} M_\env^* M_\env \big) 
& = \tr \big( C_{xx} C_{xx}^\dagger C_{xy}^*C_{xy} C_{xx}^\dagger \big) 
= \tr \big( C_{xx}^\dagger C_{xx} C_{xx}^\dagger C_{xy}^*C_{xy} \big) \\
& = \tr \big( C_{xx}^\dagger C_{xy}^*C_{xy} \big),
\end{split}
\end{equation}
where in the last equality we have applied the identity
$C_{xx}^\dagger C_{xx} C_{xx}^\dagger = C_{xx}^\dagger$.
We then observe that, again, by the closed form of the optimal 
matrix $M_\env$, we can rewrite
\begin{equation} \label{gigia}
\tr \big( C_{xy}^* M_\env \big)
= \tr \big( C_{xy}^* C_{xy} C_{xx}^\dagger \big)
= \tr \big( C_{xx}^\dagger C_{xy}^*C_{xy} \big).
\end{equation}
Substituting \cref{giulia} and \cref{gigia} in \cref{kj}, 
we get the following:
\begin{equation}
\min_{M \in \Real^{d \times k}} ~
\Exp_{(x,y) \sim \env} ~ \big \| y - M \Psi(x) \big \|^2
= \tr \big( C_{yy} \big) - \tr \big( C_{xx}^\dagger C_{xy}^*C_{xy} \big)
= \tr \big( C_{yy} \big) - \big \| C_{xx}^{1/2} M_\env \big \|_F^2,
\end{equation}
where in the last equality we have applied the optimality 
condition \cref{optimality_M_gen}. In order to terminate 
the proof, we need to prove the following equality
\begin{equation} \label{banana3}
\tr \big( C_{xx}^\dagger C_{xy}^* C_{xy} \big) 
= \tr \big( C_{yy} {\rm Corr}_{xy}^* {\rm Corr}_{xy} \big).
\end{equation}
In order to do this, we proceed as follows.
Let $L^2(\Y,\Real,\rho_\Y)$ the space of functions from $\Y$ to $\Real$ 
that are square integrable w.r.t. $\rho_\Y$ and recall that, for any $f, g \in 
L^2(\Y,\Real,\rho_\Y)$, such a space is endowed with the scalar product
\begin{equation}
\langle f, g \rangle_{L^2} = \int f(y) g(y) ~d\rho_\Y(y).
\end{equation}
Throughout the rest of the proof we will use the following operator
\begin{equation}
S: \Y \to L^2(\Y,\Real,\rho_\Y) 
\quad \quad \quad 
h \mapsto \big( y \mapsto \langle h, \cdot \rangle_\Y \big),
\end{equation}
where $\langle \cdot, \cdot \rangle_\Y$ is the scalar product in $\Y$.
Its adjoint operator $S^*:L^2(\Y,\Real,\rho_\Y) \to \Y$ is such that, for any 
$h \in \Y$ and function $f\in L^2(\Y,\Real,\rho_\Y)$,
\begin{equation}
\begin{split}
\langle h, S^*f \rangle_\Y = \langle Sh, f \rangle_{L^2}
= \int f(y) \langle h, y \rangle_\Y ~d\rho_\Y(y)
= \Bigg \langle h, \int y f(y) ~d\rho_\Y(y) \Bigg \rangle_\Y.
\end{split}
\end{equation}
This implies that, for any $f \in L^2(\Y,\Real,\rho_\Y)$, 
\begin{equation}
S^*f = \int y f(y) ~d\rho_\Y(y).
\end{equation}
In order to prove the desired statement in \cref{banana3}, we 
will use the two facts below. 

{\bf First fact.} The first fact we need is to show that the 
operator $S^*S$ coincides with the covariance operator $C_{yy}$, i.e. 
\begin{equation} \label{cov_outputs_identity}
S^*S = C_{yy}
\quad \quad \quad 
C_{yy} = \Exp ~ [y \otimes y].
\end{equation}
This fact holds, as a matter of fact, we immediately see 
that, for any $h_1, h_2 \in \Y$, we can write
\begin{equation}
\begin{split}
\big \langle h_1, S^*S h_2 \big \rangle_\Y & = \big \langle Sh_1, Sh_2 \big \rangle_{L^2}
= \int \big \langle h_1, y \big \rangle_\Y \big \langle h_2, y \big \rangle_\Y ~d\rho_\Y(y) \\
& = \Bigg \langle h_1, \Bigg(\int y\otimes y~d\rho_\Y(y)\Bigg)~h_2 \Bigg \rangle_\Y
= \big \langle h_1, C_{yy} h_2 \big \rangle_\Y.
\end{split}
\end{equation}
{\bf Second fact.} Now, recall the map $\Psi: \X \to \H$ in the 
statement and define $G:\Y \to \H$ the function
\begin{equation}
G(y) = \int \Psi(x)~d\rho(x|y)
\end{equation}
mapping $y$ into the conditional expectation of $\rho(x|y)$. Assume that $G\in L^2(\Y,\H,\rho_\Y)$, the space of functions from $\Y$ to $\H$ that are square integrable w.r.t. $\rho_\Y$. Note that $L^2(\Y, \H, \rho_\Y)$ is isometric to $\H \otimes L^2(\Y,\Real,\rho_\Y)$. Denote $J:L^2(\Y, \H, \rho_\Y) \to \H \otimes L^2(\Y,\Real,\rho_\Y)$ such an isometry and let $J_G = J(G)$ the Hilbert-Schmidt operator from $L^2(\Y,\Real,\rho_\Y)$ to $\H$ associated to $G$. Recall that the isometry follows from the observation that, given a basis $\{h_i\}_{i\in\N}$ of $\H$ and $\{f_j\}_{j\in\N}$ of $L^2(\Y,\Real,\rho_\Y)$, then the sequence $\{J(g_{ij})\}_{i,j\in\N}$, with $g_{ij}$ the vector-valued functions $g_{ij}(y) = h_i f_j(y)$ and such that $J(g_{ij}) = h_i \otimes f_j$, forms a basis for $\H \otimes L^2(\Y,\Real,\rho_\Y)$ . 

By construction, denoting by $\langle \cdot, \cdot \rangle_\H$ and 
$\langle \cdot, \cdot \rangle_{\H \otimes L^2}$ the scalar 
product in $\H$ and $\H \otimes L^2(\Y,\Real,\rho_\Y)$ respectively, 
for any $f\in L^2(\Y,\Real,\rho_\Y)$ and $h \in \H$, we have
\begin{equation}
\begin{split}
\scal{J_G}{h\otimes f}_{\H \otimes L^2} = 
\scal{G(\cdot)}{hf(\cdot)}_{L^2(\Y, \H, \rho_\Y)} = 
\int \scal{G(y)}{h}_\H f(y)~d\rho_\Y(y).
\end{split}
\end{equation}

The second fact we need is to show that the operator $J_G S$ coincides 
with the covariance operator $C_{xy}$, i.e. 
\begin{equation} \label{cov_mixed_identity}
J_G S = C_{xy}
\quad \quad \quad 
C_{xy} = \Exp ~ [y \otimes \Psi(x)]. 
\end{equation}
Also this fact holds, as a matter of fact, for any $h_1 \in \H$
and $h_2 \in \Y$, we can write the following 
\begin{equation}
\begin{split}
    \scal{h_1}{J_GSh_2}_{\H} & = \scal{J_G}{h_1\otimes(Sh_2)}_{\H \otimes L^2}\\
    & = \int \scal{G(y)}{h_1}_\H (Sh_2)(y)~d\rho_\Y(y)\\
    & = \int \scal{G(y)}{h_1}_\H \scal{h_2}{y}_\Y ~d\rho_\Y(y)\\
    & = \int \scal{h_1}{(y\otimes G(y))h_2}_\H ~d\rho_\Y(y)\\
    & = \scal{h_1}{\Bigg( \int y\otimes G(y)~d\rho_\Y(y)\Bigg)h_2}_\H \\
    & = \scal{h_1}{C_{xy} h_2}_\H,
\end{split}
\end{equation}
where in the last inequality, we have exploited 
the definition of $G$ according to which
\begin{equation}
\int y\otimes G(y)~d\rho_\Y(y) 
= \int y \otimes \Bigg(\int \Psi(x)~d\rho(x|y)\Bigg)~d\rho_\Y(y)
= \int y \otimes \Psi(x) ~d\rho(x,y)
= C_{xy}.
\end{equation}

As a consequence, recalling the covariance operator $C_{xx} = \Exp ~ [\Psi(x) \otimes \Psi(x)]$ and combining the two facts above, we can write the following steps:
\begin{equation}
\begin{split}
\tr \big( C_{xx}^\dagger C_{xy}^* C_{xy} \big) 
    & = \tr \big( C_{xx}^\dagger J_G SS^*J_G^* \big) \\
    & = \tr \big( C_{xx}^\dagger J_G SS^\dagger SS^*J_G^* \big)\\
    & = \tr \big( C_{xx}^\dagger J_G S S^\dagger S^{*\dagger} S^*S S^*J_G^* \big)\\
    & = \tr \big( C_{xx}^\dagger J_G S(S^*S)^\dagger (S^*S) S^*J_G^* \big)\\
    & = \tr \big( C_{xx}^\dagger J_G S C_{yy}^\dagger C S^*J_G^* \big)\\
    & = \tr \big( C_{yy}^\dagger C_{yy} S^*J_G^*C_{xx}^\dagger J_G S \big)\\
    & = \tr \big( C_{yy}^{\dagger/2} C_{yy} S^*J_G^*C_{xx}^\dagger J_G 
    S C_{yy}^{\dagger/2} \big)\\
    & = \tr \big( C_{yy} C_{yy}^{\dagger/2} S^*J_G^*C_{xx}^\dagger J_G S 
    C_{yy}^{\dagger/2} \big) \\
    & = \tr \big( C_{yy} C_{yy}^{\dagger/2} C_{xy} C_{xx}^\dagger C_{xy}^*
    C_{yy}^{\dagger/2} \big) \\
    & = \tr \big( C_{yy} {\rm Corr}_{xy}^* {\rm Corr}_{xy} \big),
\end{split}
\end{equation}
where, in the first equation we have used \cref{cov_mixed_identity},
in the second, third and fourth equality we have used the following 
standard relations
\begin{equation}
S = SS^\dagger S 
\quad \quad 
S^\dagger = S^\dagger S^{*\dagger} S^*
\quad \quad  
(S^*S)^\dagger = S^\dagger S^{*\dagger},
\end{equation}
in the fifth equality we have used \cref{cov_outputs_identity},
in the eighth equality we have exploited the commuting property 
$C_{yy}^{\dagger/2} C_{yy} = C_{yy} C_{yy}^{\dagger/2}$,
in the ninth equality we have used again \cref{cov_mixed_identity}
and, finally, in the last equality, we have introduced the definition
of the correlation operator
\begin{equation}
{\rm Corr}_{xy} = C_{xx}^{\dagger/2}  C_{xy} C_{yy}^{\dagger/2},
\end{equation}
which is used in Canonical Correlation Analysis.
\end{proof}

We now have all the ingredient for the proof of \cref{oracle_meta_parameters}.

\begin{proof}{\bf of \cref{oracle_meta_parameters}.}
We start from recalling the problem we want to solve:
\begin{equation} 
\min_{M \in \Real^{d \times k}, b \in \Real^d}
~ {\rm Var}_\env(\tau_{M,b})^2  
= \min_{M \in \Real^{d \times k}, b \in \Real^d} ~
\Exp_{(\task, s) \sim \env} ~ \big \| w_\task - (M \Phi(s) + b) \big \|^2.
\end{equation}
By taking the derivatives w.r.t. $b$, we conclude that the matrix 
$M_\env \in \Real^{d \times k}$ and the vector $b_\env \in \Real^d$ 
minimizing the term above satisfy 
\begin{equation}
w_\env = M_\env ~ \nu_\env + b_\env,
\end{equation}
or, equivalently,
\begin{equation}
b _\env= w_\env - M_\env ~ \nu_\env.
\end{equation}
Exploiting this equality, we can rewrite our problem above as
\begin{equation*}
\min_{M \in \Real^{d \times k}, b \in \Real^d} ~
\Exp_{(\task, s) \sim \env} ~ \big \| w_\task - (M \Phi(s) + b) \big \|^2
= \min_{M \in \Real^{d \times k}} ~ \Exp_{(\task, s) \sim \env} ~ 
\big \| (w_\task - w_\env) - M \bigl( \Phi(s) - \nu_\env \bigr) \big \|^2.
\end{equation*}
We now observe that the problem above has the same form 
of the problem considered in \cref{extracting_cov}, once one
identifies $\X = \Ss$ (the space of the side information), 
$x = s$, $y = w_\task - w_\env$ and $\Psi(x) = \Phi(s)
- \nu_\env$. The desired statements automatically derive
from the application of \cref{extracting_cov} to our context.
\end{proof}


\section{Proofs of the statements in \cref{proposed_method}}
\label{proofs_proposed_method_sec}

In this section we report the proofs of the statements we used
in \cref{proposed_method} in order to prove the expected excess 
risk bound for \cref{OGDA2_paper} in \cref{bound_estimated_bias}.
We start from proving in \cref{properties_surrogate_proof}
the properties of the surrogate functions in \cref{properties_surrogate}.
Then, in \cref{proof_conv_rate_surr}, we give the convergence rate 
of \cref{OGDA2_paper} on the surrogate problem in \cref{surrogate_linear}.
We conclude by describing in \cref{implementation_kernels}
how \cref{OGDA2_paper} can be implemented
by computing only evaluations of the kernel associated to the 
feature map $\Phi$, without the need of 
explicitly evaluating the feature map itself. This is useful when the 
space in which the image of the feature map lies is high (or even infinite)
dimensional.

\subsection{Proof of \cref{properties_surrogate}}
\label{properties_surrogate_proof}

We now prove the properties of the surrogate functions in 
\cref{properties_surrogate}.

\PropertiesSurrogate*

\begin{proof} 
We are interested in studying the properties of the surrogate function 
\begin{equation}
\LL \big (\cdot, \cdot, s, \Zn \big ): 
\Real^{d \times k} \times \Real^d \to \Real
\end{equation}
in \cref{surrogate_linear}. We start from observing that, 
such a function coincides with the composition of the Moreau envelope  
$\hat \Delta(\cdot, \Zn): \Real^d \to \Real$ of the empirical risk $\cR_\Zn$:
\begin{equation} \label{banana5}
\theta \mapsto 
\hat \Delta(\theta, \Zn) = \min_{w \in \Real^d} ~ \cR^\la_\Zn(w)
\quad \quad \quad 
\cR_\Zn^\la(w) = \frac{1}{n} \sum_{i = 1}^n \ell(\langle x_i, w \rangle, y_i) 
+ \frac{\la}{2} \| w - \theta \|^2
\end{equation}
with the linear transformation 
\begin{equation}
s \in \Ss \mapsto \tau_{M,b}(s) = M \Phi(s) + b \in \Real^d.
\end{equation}
In other words, for any $M \in \Real^{d \times k}$ and $b \in \Real^d$,
we can write
\begin{equation}
\LL \big (M, b, s, \Zn \big ) = \hat \Delta(\tau_{M,b}(s), \Zn).
\end{equation}
As a consequence, since the Moreau envelope is convex and differentiable
\cite[Prop. $12.29$]{bauschke2011convex}, the resulting surrogate function 
$\LL \big (\cdot, \cdot, s, \Zn \big )$
is convex and differentiable over $\Real^{d \times k} \times \Real^d$.
The closed form of the gradient in \cref{gradient1} 
directly derives from the composition rule for derivatives and the closed
form of the gradient of the Moreau envelope \cite[Prop. $12.29$]{bauschke2011convex}
\begin{equation} \label{banana4}
\nabla \hat \Delta(\cdot, \Zn) (\theta) = 
- \la \big( A (\theta, \Zn ) - \theta \big) \in \Real^d,
\end{equation}
with $A(\theta, \Zn)$ defined as in \cref{RERM_bias}.
Consequently, we get 
\begin{equation}
\nabla \LL\big (\cdot, \cdot, s, \Zn \big )(M, b) =
\nabla \hat \Delta(\cdot, \Zn) (\tau_{M, b}(s))
\left(\begin{array}{c}\Phi(s)\\ 1\end{array}\right) \trans,
\end{equation}
coinciding with the desired closed form in \cref{gradient1}.
Finally, we observe that, as shown in \cite[Prop. $4$]{denevi2019learning}, under \cref{ass_1}, for any $\theta \in \Real^d$, we have
\begin{equation}
\big \| \nabla \hat \Delta(\cdot, \Zn) (\theta) \big \|^2 \le L^2 \rx^2.
\end{equation}
As a consequence, exploiting the rewriting above, 
\cref{ass_1} and \cref{ass_3}, we get the desired bound in \cref{bound_grad_norm}:
\begin{equation}
\begin{split}
\big \| \nabla \LL\big (\cdot, \cdot, s, \Zn \big )(M, b) \big\|_F^2
& = \big \| \nabla \hat \Delta(\cdot, \Zn) \big(\tau_{M,b}(s)\big) 
\Phi(s) \trans \big \|_F^2 
+ \big \| \nabla \hat \Delta(\cdot, \Zn) \big(\tau_{M,b}(s)\big) \big \|^2 \\
& = \big \| \nabla \hat \Delta(\cdot, \Zn) \big(\tau_{M,b}(s)\big) \big \|^2 
\big \| \Phi(s) \big \|^2 
+ \big \| \nabla \hat \Delta(\cdot, \Zn) \big(\tau_{M,b}(s)\big) \big \|^2 \\
& \le L^2 \rx^2 (K^2 + 1),
\end{split}
\end{equation}
where in the second equality above we have exploited the fact that for any 
vectors $a \in \Real^d$ and $b \in \Real^s$, we have
\begin{equation}
\big \| a b \trans \big \|_F^2 
= \tr \big( b a \trans a b \trans \big) 
= \tr \big( b \trans b a \trans a \big) 
= \| a \|^2 \| b \|^2.
\end{equation}

\end{proof}

\subsection{Convergence rate of \cref{OGDA2_paper} on the surrogate 
problem in \cref{surrogate_linear}}
\label{proof_conv_rate_surr}

We now give the convergence rate of \cref{OGDA2_paper} on the surrogate
problem in \cref{surrogate_linear}.

\begin{restatable}[Convergence rate on the surrogate problem in \cref{surrogate_linear}]{proposition}{ConvergenceSurrogate} \label{convergence_surrogate}
Let $\thickbar M$ and $\thickbar b$ be the average of the iterations obtained from the application of \cref{OGDA2_paper} over the training data $(\Zn_t, s_t)_{t = 1}^T$ with constant meta-step size $\gamma > 0$ and inner regularization parameter $\la > 0$. Then, under \cref{ass_1} and \cref{ass_3}, for any $\tau_{M,b} \in \T_\Phi$, in expectation w.r.t. the sampling 
of $(\Zn_t, s_t)_{t = 1}^T$,
\begin{equation} \label{B_bias_1}
\Exp ~ \hat \ee_\env \bigl(\tau_{\thickbar M, \thickbar b} \bigr) 
- \hat \ee_\env \bigl( \tau_{M, b} \bigr) \le \frac{\gamma L^2 \rx^2 (K^2 + 1)}{2} 
+ \frac{\big\| (M, b) \big\|_F^2}{2 \gamma T}.
\end{equation}
\end{restatable}

\begin{proof}
We observe that \cref{OGDA2_paper} coincides with Stochastic
Gradient Descent applied to the convex and Lipschitz (see \cref{properties_surrogate})
surrogate problem in 
\cref{surrogate_linear}:
\begin{equation}
\min_{M \in \Real^{d \times k}, b \in \Real^d} ~ \hat \ee_\env(\tau_{M, b})
\quad \quad 
\hat \ee_\env(\tau_{M, b}) = \Exp_{(\task, s) \sim \env} 
~ \Exp_{\Zn \sim \task^n} ~ \LL\big (M, b, s, \Zn \big ).
\end{equation}
As a consequence, by standard arguments
(see e.g. \cite[Lemma $14.1$, Thm. $14.8$]{shalev2014understanding} 
and references therein), for any $\tau_{M,b} \in \T_\Phi$, we have
\begin{equation}
\Exp~ \hat \ee_\env \bigl(\tau_{\thickbar M, \thickbar b} \bigr) 
- \hat \ee_\env \bigl( \tau_{M, b} \bigr) \le \frac{\gamma}{2 T} 
\sum_{t = 1}^T \Exp ~ \big \| \nabla \LL\big (\cdot, \cdot, 
s, \Zn_t \big )(M_t, b_t) \big\|_F^2
+ \frac{\big\| (M, b) \big\|_F^2}{2 \gamma T}.
\end{equation}
The desired statement derives from combining this 
bound with the bound on the norm of the meta-subgradients
in \cref{bound_grad_norm} in \cref{properties_surrogate}.
\end{proof}

\subsection{Implementation of \cref{OGDA2_paper} with kernels}
\label{implementation_kernels}

We conclude this section by describing how \cref{OGDA2_paper} 
can be implemented by computing only evaluations of the kernel 
associated to the feature map $\Phi$.
We describe this in the following lemma exploiting standard arguments 
from online learning with kernels literature
(see e.g. \cite{kivinen2004online,singh2012online,shalev2014understanding}).

\begin{lemma}[Implementation of \cref{OGDA2_paper} by kernel's evaluations] \label{lemma_only_kernel}
Let $\bigl(M_t, b_t, \theta_t \big)_{t = 1}^T$ be the iteration generated 
by \cref{OGDA2_paper} with meta-step size $\gamma \ge 0$. Then, 
\begin{equation}
\theta_{t+1} = - \gamma \sum_{j = 1}^t 
\nabla \hat \Delta(\cdot, \Zn_j) (\tau_{M_j, b_j}(s_j))
~ k(s_j, s_{t+1}) + b_{t+1},
\end{equation}
where the function $\hat \Delta$ and its gradients 
$\nabla \hat \Delta(\cdot, \Zn_j)$ are defined in \cref{banana5}
and \cref{banana4} above and we have introduced 
the evaluation 
\begin{equation}
k(s_j, s_{t+1}) = \Phi(s_j) \trans \Phi(s_{t+1}),
\end{equation}
of the kernel associated to the feature map $\Phi$.
\end{lemma}

\begin{proof}
Exploiting the closed form of the meta-subgradient 
in \cref{gradient1} in \cref{properties_surrogate}, we
can rewrite more explicitly the update step of 
\cref{OGDA2_paper} as follows:
\begin{equation}
\begin{split}
& M_{t+1} = M_t - \gamma ~ \nabla \hat \Delta(\cdot, \Zn_t) (\tau_{M_t, b_t}(s_t))
\Phi(s_t) \trans \\
& b_{t+1} = b_t - \gamma ~ \nabla \hat \Delta(\cdot, \Zn_t) (\tau_{M_t, b_t}(s_t)) \\
& \theta_{t + 1} = M_{t+1} \Phi(s_{t+1}) + b_{t+1}.
\end{split}
\end{equation}
By induction argument on the iteration $t$, one can easily see that 
the update of the matrix $M_{t+1}$ can be equivalently rewritten as
\begin{equation}
M_{t+1} = - \gamma \sum_{j = 1}^t \nabla \hat \Delta(\cdot, \Zn_j) 
(\tau_{M_j, b_j}(s_j))
\Phi(s_j) \trans.
\end{equation}
As a consequence, we can rewrite the update of the bias vector $\theta_{t+1}$
as follows
\begin{equation}
\begin{split}
\theta_{t+1} & = M_{t+1} \Phi(s_{t+1}) + b_{t+1} \\
& = - \gamma \sum_{j = 1}^t \nabla \hat \Delta(\cdot, \Zn_j) (\tau_{M_j, b_j}(s_j)) 
\Phi(s_j) \trans \Phi(s_{t+1}) + b_{t+1} \\
& = - \gamma \sum_{j = 1}^t \nabla \hat \Delta(\cdot, \Zn_j) (\tau_{M_j, b_j}(s_j))
~ k(s_j, s_{t+1}) + b_{t+1}.
\end{split}
\end{equation}
This last equation coincides with the desired statement.
\end{proof}



\section{Additional real experiments and experimental details}
\label{exp_details}

In the first part of this section we report two additional real experiments,
in the second part we report the implementation details we omitted in
the main body.

\subsection{Additional real experiments}
\label{add_exps}

\begin{figure}[t]
\begin{minipage}[t]{0.49\textwidth}  
\centering
\includegraphics[width=1\textwidth]{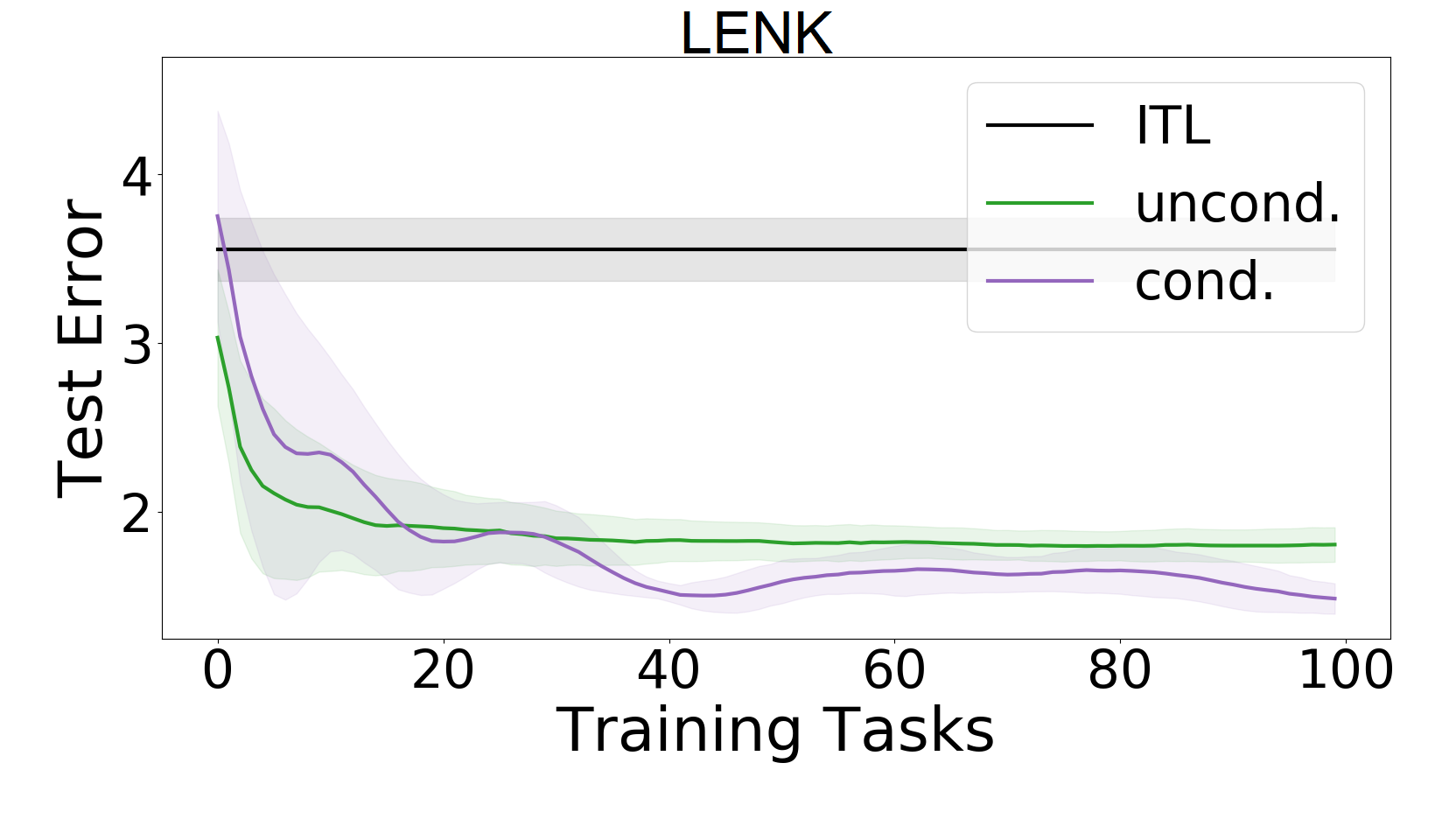}
\end{minipage}
\begin{minipage}[t]{0.49\textwidth}
\centering
\includegraphics[width=1\textwidth]{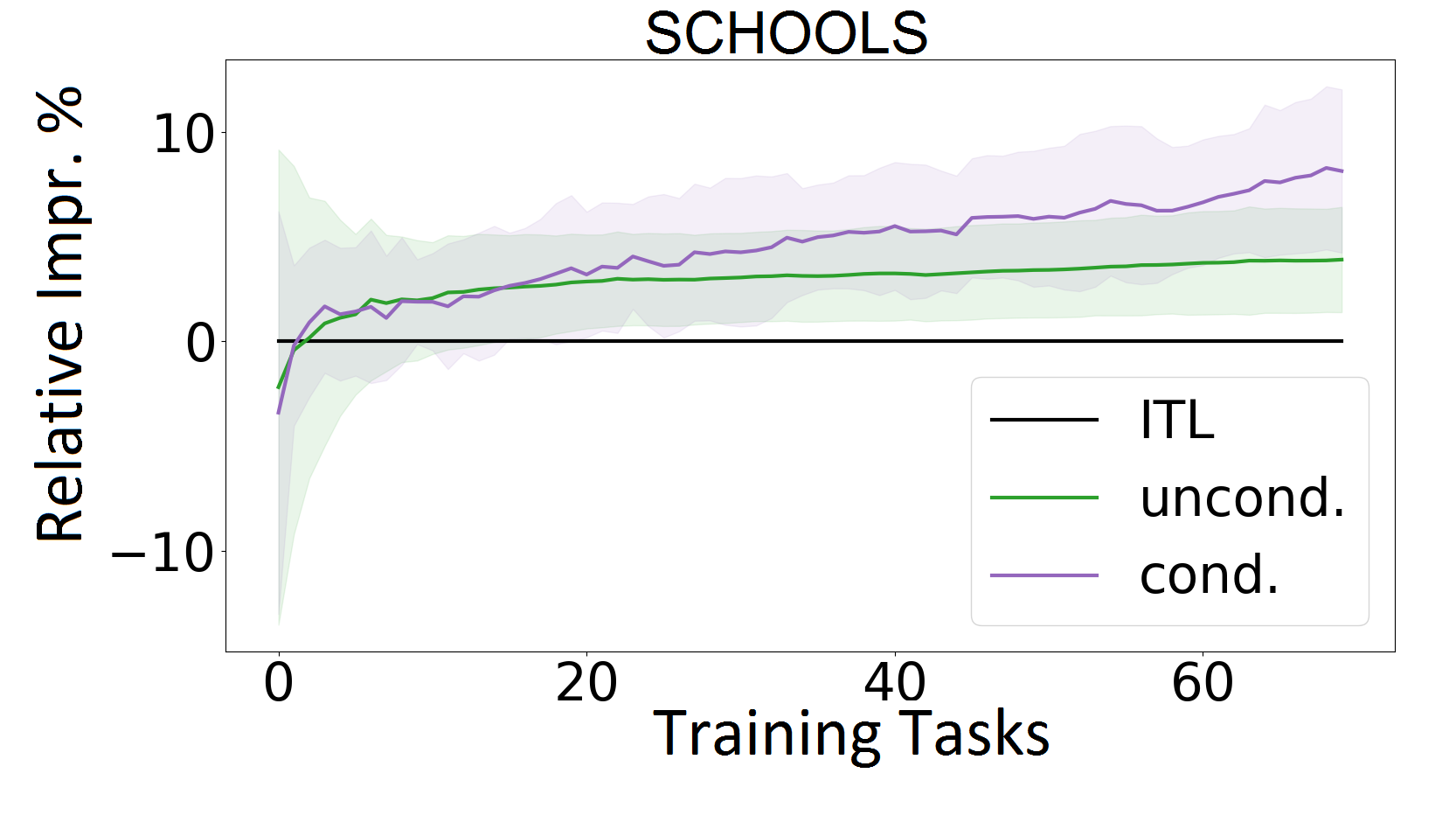} \\ 
\end{minipage}
\caption{Performance (averaged over $10$ seeds) 
of different methods w.r.t. an increasing number of tasks.
Lenk dataset (left), Schools dataset (right).
\label{fig_exps_real}}
\end{figure}

We tested the performance of our method also on two regression 
problems on the Lenk and the Schools datasets. Also in these cases, 
we evaluated the errors by the absolute loss and we implemented
the variant of the methods with the online inner algorithm in \cref{online_inner_algorithm}. We used again as side information 
datapoints. \\

\paragraph{Lenk dataset.} We considered the computer survey data from \cite{lenk1996hierarchical,Andrew}, in which $T_{\rm tot} = 180$ people 
(tasks) rated the likelihood of purchasing one of $n_{\rm tot} = 20$ different
personal computers. The input represents $d = 13$ different computers'
characteristics, while the output is an integer rating from $0$ to $10$. 
\cref{fig_exps_real} (left) shows that, coherently to previous literature
\cite{denevi2019learning}, the unconditional approach significantly 
outperforms ITL, but the performance of its conditional counterpart is 
even better. \\

\paragraph{Schools dataset.}
We considered the Schools dataset \cite{argyriou2008convex}, consisting of 
examination records from $T_{\rm tot} = 139$ schools. Each school is 
associated to a task, individual students are represented by a 
features' vectors $x \in \Real^d$, with $d = 26$, and their exam scores 
to the outputs. The sample size $n_{\rm tot}$ varies across the 
tasks from a minimum $24$ to a maximum $251$. 
\cref{fig_exps_real} (right) shows that, also in this case, the unconditional 
approach brings a meaningful improvement w.r.t. ITL, but the gain provided
by its conditional counterpart is even more evident.

\subsection{Experimental details}
\label{exps_details_sub}

In order to tune the hyper-parameters $\la$ and $\gamma$ our experiments, 
we followed the same validation procedure described 
in \cite[App. I]{denevi2019learning}. Such a procedure requires performing 
a meta-training, a meta-validation and a meta-test phase on a separate sets of
$T_{\rm tr}$ training tasks, $T_{\rm va}$ validation tasks and $T_{\rm te}$ 
test tasks. Each task in the training set is observed by a corresponding dataset 
$\Zn_{\rm tr}$ of $n = n_{\rm tr}$ points, while, the tasks in the test and 
validation sets are all provided with a corresponding training dataset $\Zn_{\rm tr}$ 
of $n_{\rm tr}$ points and a corresponding test dataset $\Zn_{\rm te}$ of 
$n_{\rm te}$ points.

Specifically, in our experiments, we applied the validation procedure above as 
described in the following.\\

\paragraph{Synthetic clusters} We considered $14$ candidates values for both $\la$ and $\eta$ in the range $[10^{-5}, 10^5]$ with logarithmic spacing and we evaluated the performance of the estimated feature maps by using $T =T_{\rm tr} = 300$, $T_{\rm va} = 100$, $T_{\rm te} = 80$ of the available tasks for meta-training, meta-validation and meta-testing, respectively. In order to train and to test the inner algorithm, we splitted each within-task dataset into $n = n_{\rm tr} = 50\% ~ n_{\rm tot}$ for training and $n_{\rm te} = 50\% ~ n_{\rm tot}$ for test. We implemented our conditional method using as side information the input points $X = (x_i)_{i = 1}^n \in \bigcup_{n\in\N}\X^n$ and the feature map $\Phi:\bigcup_{n\in\N}\X^n \to \Real^d$ defined by $\Phi(X) = \frac{1}{n} \sum_{i = 1}^n x_i$. \\

\paragraph{Synthetic circle} We considered $16$ candidates values for both $\la$ and $\eta$ in the range $[10^{-7}, 10^7]$ with logarithmic spacing and we splitted the data as in the
clusters' settings above. As already spoiled in the main body, we applied our conditional 
approach with two different feature maps: the true underlying feature map $\Phi(s) = 
({\rm cos}(2 \pi s), {\rm sin}(2 \pi s))$ and the feature map mimicking a Gaussian 
distribution by Fourier random features described below (at the end of this section) with parameters $k = 50$ and $\sigma = 10$.\\

\paragraph{Lenk dataset}
We considered $14$ candidates values for both $\la$ and $\eta$ in the range $[10^{-5}, 10^5]$ with logarithmic spacing and we evaluated the performance of the estimated feature maps by splitting the tasks into $T =T_{\rm tr} = 100$, $T_{\rm va} = 40$, $T_{\rm te} = 30$ tasks used for meta-training, meta-validation and meta-testing, respectively. In order to train and to test the inner algorithm, we splitted each within-task dataset into $n = n_{\rm tr} = 16$ for training and $n_{\rm te} = 4$ for test. We used as side information the datapoints $\Zn =(z_i)_{i = 1}^n$ and the feature map $\Phi:\D \to \Real^{2d}$ defined by $\Phi(\Zn) = \frac{1}{n} \sum_{i = 1}^n \phi(z_i)$, with $\phi(z_i) = {\rm vec}\big( x_i (y_i, 1) \trans \big)$, where, for any matrix $A = [a_1, a_2] \in \Real^{d \times 2}$ with columns $a_1, a_2 \in \Real^d$, ${\rm vec}(A) = (a_1, a_2) \trans \in \Real^{2d}$.\\

\paragraph{Schools dataset}
We considered $14$ candidates values for both $\la$ and $\eta$ in the range $[10^{-5}, 10^5]$ with logarithmic spacing and we evaluated the performance of the estimated feature maps by splitting the tasks into $T =T_{\rm tr} = 70$, $T_{\rm va} = 39$, $T_{\rm te} = 30$ tasks used for meta-training, meta-validation and meta-testing, respectively. In order to train and to test the inner algorithm, we splitted each within-task dataset into $n = n_{\rm tr} = 75\% ~ n_{\rm tot}$ for training and $n_{\rm te} = 25\% ~ n_{\rm tot}$ for test. We used as side information the inputs $X = (x_i)_{i = 1}^n$ and the feature map mimicking a Gaussian distribution by Fourier random features described below (at the end of this section) with parameters $k = 1000$ and 
$\sigma = 100$.\\

\paragraph{Feature map by Fourier random features}
We now describe the feature map mimicking 
a Gaussian distribution by Fourier random features \cite{rahimi2008random} 
we used in our synthetic circle experiment and Schools dataset experiment. 
We recall that, in these cases, we considered as side information the  
inputs $X = (x_i)_{i = 1}^n$. The feature map above was then defined as 
$\Phi(X) = \frac{1}{n} \sum_{i = 1}^n \phi(x_i)$, where, $\phi$ was 
built as follows. 
We first introduced an integer $k \in \N$ and a constant $\sigma \in \Real$. 
We then sampled a vector $v \in \Real^k$ from the uniform distribution over 
$[0,2 \pi]^k$ and a matrix $U \in \Real^{k \times d}$ is 
sampled from the Gaussian distribution $\mathcal{N}(0,\sigma I)$. 
We then defined
\begin{equation}
\phi(x_i) = \sqrt{\frac{2}{k}} ~ {\rm cos}\big(U x_i + v \big) \in \Real^k,
\end{equation}
where ${\cos}(\cdot)$ is applied component-wise to the vector.\\

We conclude this section reporting the characteristics of the machine 
we used for running our experiments and the complexity of our method 
in \cref{OGDA2_paper}.

All the experiments were conducted on a workstation with 4 Intel Xeon E5-2697 V3 2.60Ghz CPUs and 256GB RAM.

The variant of our method in \cref{OGDA2_paper} for biased regularization using the batch inner algorithm in \cref{RERM_bias} has a time and space complexity $\mathcal{O}(d (k+n))$. The variant for fine-tuning using the online inner algorithm in \cref{online_inner_algorithm} has a time and space complexity $\mathcal{O}(dk)$.

\end{document}